\documentclass{article}

\usepackage[preprint]{neurips_2026}

\usepackage[utf8]{inputenc} % allow utf-8 input
\usepackage[T1]{fontenc}    % use 8-bit T1 fonts
\usepackage{hyperref}       % hyperlinks
\usepackage{url}            % simple URL typesetting
\usepackage{booktabs}       % professional-quality tables
\usepackage{amsfonts}       % blackboard math symbols
\usepackage{nicefrac}       % compact symbols for 1/2, etc.
\usepackage{microtype}      % microtypography
\usepackage{xcolor}         % colors
\usepackage{xspace}
\usepackage{enumitem}
\usepackage[pdftex]{graphicx}
\usepackage{subcaption}
\usepackage{amsmath}
\usepackage[capitalize,noabbrev]{cleveref}
\usepackage{multirow}
\usepackage{tcolorbox}
\usepackage{adjustbox}
\usepackage{makecell}
\definecolor{lightred}{RGB}{255,230,230}
\usepackage{colortbl}

\usepackage{amsthm}
\theoremstyle{plain}
\newtheorem{theorem}{Theorem}[section]
\newtheorem{proposition}[theorem]{Proposition}

\theoremstyle{definition}
\newtheorem{definition}[theorem]{Definition}
\newtheorem{assumption}[theorem]{Assumption}
\theoremstyle{remark}

\newcommand{\sys}{\texttt{Quant.npu}\xspace}

\title{Quant.npu: Enabling Efficient Mobile NPU Inference for on-device LLMs via Fully Static Quantization}

\author{%
  \textbf{Jinghe Zhang, Daliang Xu\thanks{Corresponding author}, Chenghua Wang, Weikai Xie, Tao Qi,} \\
  \textbf{Yun Ma, Mengwei Xu, Gang Huang} \\
}

\begin{document}

\maketitle

\begin{abstract}
  Large language models (LLMs) are increasingly deployed on mobile devices, where Neural Processing Units (NPUs) necessitate fully static quantization for optimal inference efficiency. However, existing post-training quantization (PTQ) methods predominantly rely on dynamic activation quantization, rendering them incompatible with NPU hardware constraints. To bridge the gap between high-fidelity PTQ and NPU-constrained inference, we propose \sys, a integer-only fully static quantization framework. It incorporates learnable quantization parameters and rotation matrices, enabling low-bit activation-weight quantization without runtime quantization parameters re-computation. Crucially, we identify that initialization and selective optimization of quantization parameters is pivotal for optimization stability, as improper initialization and naive joint optimization induce gradient instability that disrupts the optimization of rotation matrices. To address this, we propose a rotation-and-bit-width-aware initialization tailored to diverse activation profiles and a distribution-aware selective optimization (two-stage quantization pipeline) tailored to rotated and unrotated tensors. Furthermore, we introduce a sensitivity-guided adaptive mixed-precision scheme to balance accuracy with inference efficiency.  Extensive experiments on real-world mobile NPUs demonstrate that \sys achieves comparable accuracy to state-of-the-art methods, while reducing inference latency by up to 15.1\%.
\end{abstract}

\section{Introduction} 
On-device LLM inference has emerged as a crucial research direction, primarily driven by its inherent advantages in safeguarding data privacy and enabling low-latency, network-independent operations~\cite{shafee2025privacy, pamadi2025edge}. 
To enable the energy-efficient execution on mobile devices, manufacturers integrate Neural Processing Units (NPUs), an LLM-specialized hardware into their System-on-Chips (SoCs). 
For instance, Qualcomm's Hexagon NPU features dedicated matrix units for high-throughput GEMM, vector units for element-wise processing, and specialized direct memory access (DMA) engines to maximize memory bandwidth utilization. These architectural enhancements deliver more than 10$\times$ performance and 4$\times$ energy efficiency than mobile GPUs~\cite{xu2024fastondevicellminference}.

These efficiency gains, however, impose specific architectural constraints:
(i) Preference for integer arithmetic: NPUs prioritize high-throughput integer matrix multiplication due to the superior area and energy efficiency of integer units compared to floating-point counterparts~\cite{xu2024fastondevicellminference}.
(ii) Requirement for static quantization: NPUs are architected for static quantization~\cite{qualcomm_applyencodings_2026} to avoid the high computational overhead of dynamic reduction operations (e.g., on-the-fly min/max computation). For instance, performing a tree reduction on a 128-element INT8 vector requires at least 7 vector instruction cycles, which creates a bottleneck.
(iii) Inefficiency of fine-grained quantization: NPUs generally rely on systolic array-based cores, favoring coarse-grained quantization~\cite{qualcomm_applyencodings_2026}. Fine-grained methods (e.g., per-block) incur frequent de-quantization overhead, disrupting the pipeline and reducing throughput. On Hexagon NPUs, standard W4A8 per-tensor inference is $\approx$20\% faster than W4A16 per-block.

However, prevailing quantization paradigms~\cite{frantar2023gptqaccurateposttrainingquantization,xiao2024smoothquantaccurateefficientposttraining,liu2025spinquantllmquantizationlearned, sun2025flatquantflatnessmattersllm}, which are predominantly tailored for \textit{GPUs and dynamic quantization}, fundamentally struggle to align with mobile NPUs. While Post-Training Quantization (PTQ) methods attempt to mitigate accuracy loss through second-order optimization~\cite{frantar2023gptqaccurateposttrainingquantization} and channel-wise outlier suppression~\cite{xiao2024smoothquantaccurateefficientposttraining}, their fully static variants remain highly vulnerable to the extreme activation outliers inherent in LLMs, leading to severe accuracy degradation. To address this, recent rotation-based approaches~\cite{liu2025spinquantllmquantizationlearned, sun2025flatquantflatnessmattersllm} employ learnable rotations to smooth activation distributions and mitigate the impact of outliers. However, these methods typically assume support for dynamic activation quantization during deployment and use dynamic quantization simulation during optimization. This assumption conflicts with the static compilation constraints of mobile NPUs, often causing severe accuracy collapse when models are converted to a static setting.

To address this, this paper integrates fully static quantization with learnable quantization parameters and rotation matrices, bridging the gap between optimization and deployment to prevent performance degradation. However, achieving this is non-trivial. The primary distinction between static and dynamic quantization lies in the accuracy of quantization parameter estimation. Inappropriate settings introduce significant quantization error, causing the optimization process to converge slowly or fail. Our preliminary experiments (~\cref{sec: init motivation}) confirm that these settings critically impact convergence stability: (i) Initialization of activation quantization parameters: The initial values determines the clipping range. Suboptimal initialization leads to either excessive clipping distortion (if the range is too narrow) or wasted bit-width resolution (if the range is too wide), creating a poor optimization landscape and insufficient convergence. (ii) Selectivity in learnable quantization parameters: We observe that jointly optimizing the quantization parameters for every tensor can be counterproductive.

\textbf{Our solution.} To address these challenges, we propose \sys, an algorithm-system co-designed system with full static quantization, tailored for practically efficient on-NPU LLM inference.
It incorporates three key techniques:
(i) Rotation-and-bit-width-aware Initialization: We demonstrate that the convergence accuracy of typical initialization strategies (e.g., Max-Min vs. Mean) varies depending on the target bit-width and the presence of rotations. Our method adapts the initialization accordingly to ensure a stable starting point.
(ii) Distribution-Aware Selective Optimization (two-stage quantization pipeline): We effectively decouple the optimization of rotated and unrotated distributions. Since unrotated distributions are harder to optimize and increase complexity, we treat them distinctly. 
(iii) Layer-wise Adaptive Mixed-Precision Strategy: Guided by a quantization sensitivity metric, we selectively assign higher bit-widths (e.g., 16-bit) to sensitive components. This allows \sys to employ high precision where necessary, reducing inference latency overhead.

We evaluate \sys on a real mobile NPU (Qualcomm SM8650) using 4 mobile-sized (1--3B) LLMs across 3 real scenario and 6 accuracy datasets. To ensure zero inference overhead, we introduce only two rotation matrices ($R_1, R_2$) that are fused into weight matrices offline, as shown in ~\cref{fig:framework}. To fully leverage hardware capabilities, \sys employs per-tensor quantization for activations and per-channel quantization for weights.
At identical inference latency, our method consistently achieves higher accuracy on downstream tasks than existing approaches. Furthermore, compared to state-of-the-art approaches, which are up to 15.1\% slower than our method, \sys incurs only a marginal performance gap, with an average accuracy drop of 2.58\% and a PPL increase of 1.23.

Our main contributions are summarized as follows:
\begin{itemize}[nosep, leftmargin=1.5em]
    \item \textbf{Identification:} We find that the optimization stability of fully static quantization is highly sensitive to the initialization and the selection of learnable quantization parameters.
    \item \textbf{Framework:} We propose a mobile NPU-friendly quantization framework that incorporates novel rotation-and-bit-width-aware initialization and distribution-aware selective optimization to effectively align optimization process and deployment under static NPU constraints.
    \item \textbf{Performance:} Our method achieves comparable SOTA accuracy with up to 15.1\% lower latency.
\end{itemize}

\begin{figure*}[t]
  \centering
  \begin{subfigure}[t]{0.32\linewidth}
    \centering
    \includegraphics[width=\linewidth]{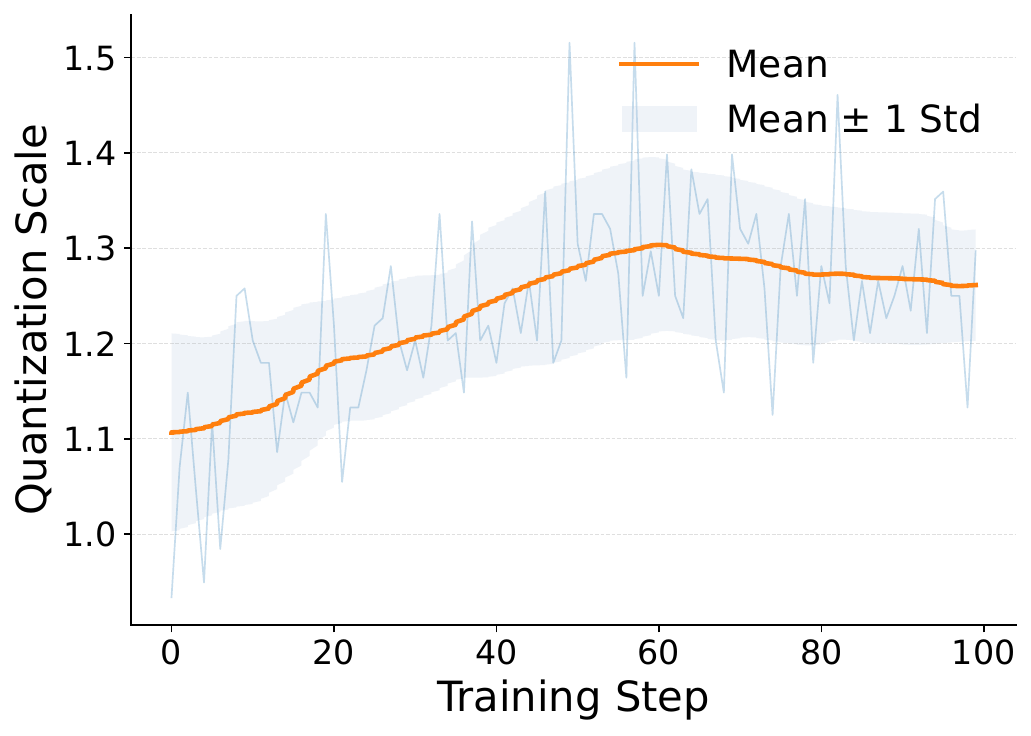}
    \caption{Fluctuation of dynamic activation scale for $W_o$ of the $21^{th}$ layer.}
    \label{fig:a}
  \end{subfigure}
  \hfill
  \begin{subfigure}[t]{0.32\linewidth}
    \centering
    \includegraphics[width=\linewidth]{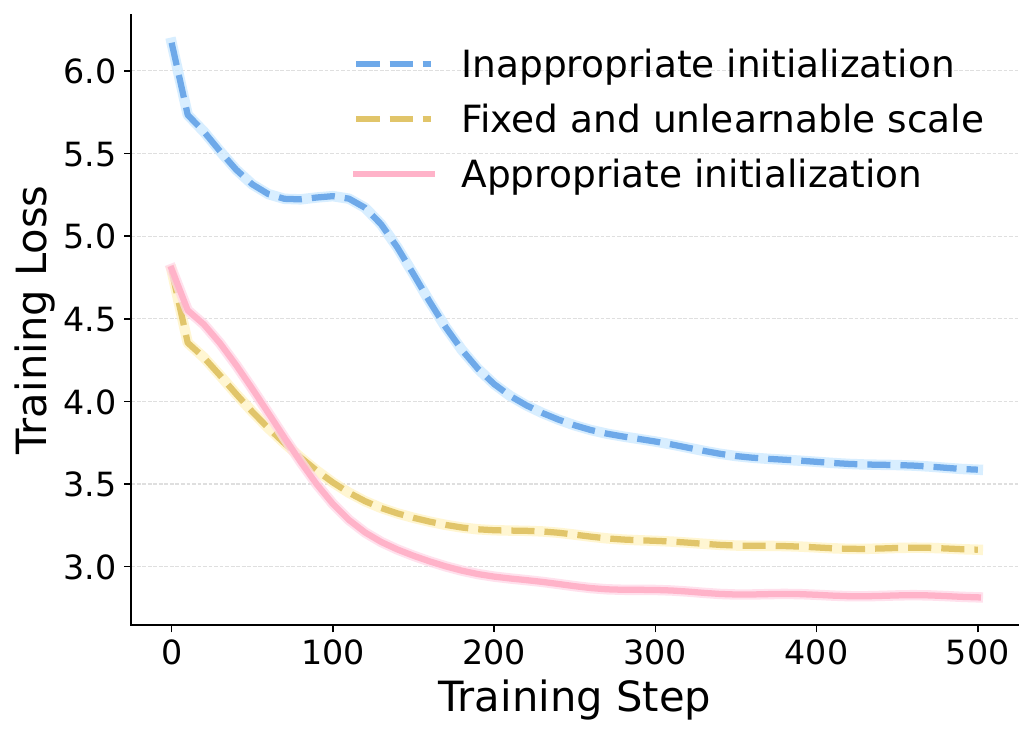}
    \caption{Loss convergence curves under different optimization settings.}
    \label{fig:b}
  \end{subfigure}
  \hfill
  \begin{subfigure}[t]{0.32\linewidth}
    \centering
    \includegraphics[width=\linewidth]{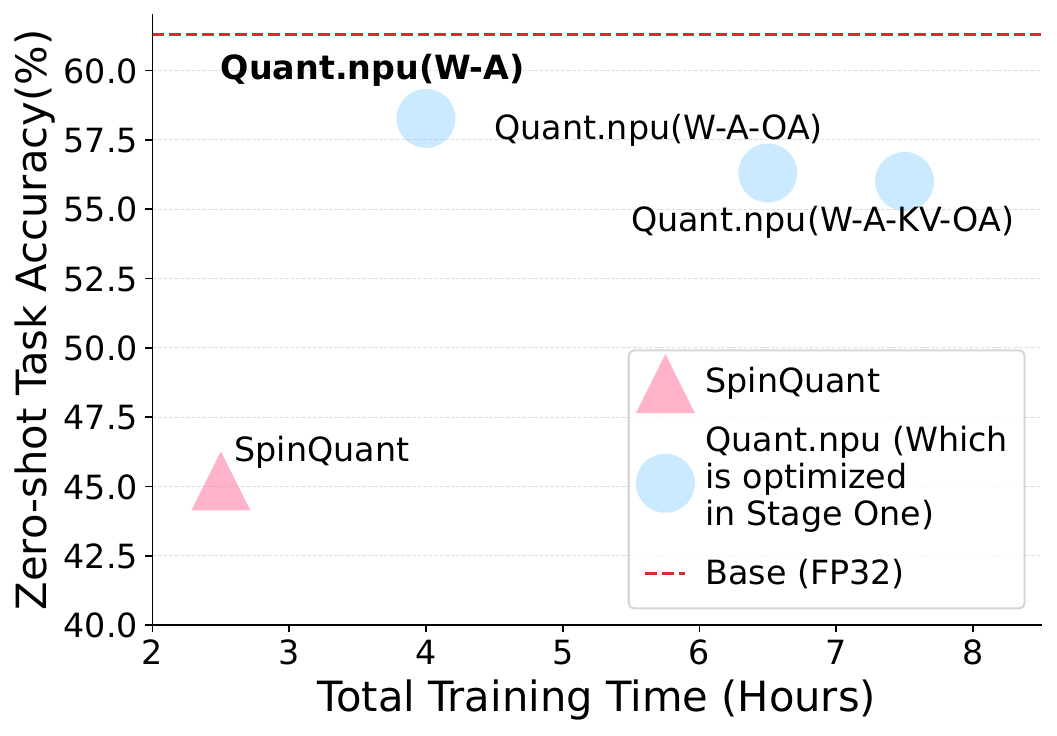}
    \caption{Different components in Stage One impacts on accuracy and time.}
    \label{fig:c}
  \end{subfigure}

  \caption{The influence of scale on training. \cref{fig:a} shows large fluctuations in dynamic scale during optimization. \cref{fig:b} indicates that the learnability and initialization of quantization parameters are crucial for convergence. In \cref{fig:c}, "A", "OA", and "W" denote the input activation, output activation, and weights of the linear layers. "KV" refers to the Key and Value matrices.
  }
  \label{fig:abc}
\end{figure*}

\section{Background}
\subsection{LLM Inference on Mobile NPUs}
Modern SoCs increasingly integrate NPUs to accelerate LLM inference~\cite{10592049} by optimizing integer arithmetic, offering a $2.0\times$ speedup over standard floating-point operations (\cref{tab:mm_overhead} in Appendix). To further maximize efficiency, hardware architectures like Qualcomm NPUs favor coarse-grained quantization schemes (per-tensor activation and per-channel weight) over fine-grained per-block alternatives, yielding an additional $1.2\times$ speedup. However, this efficient coarse-grained scheme causes significant accuracy degradation (e.g., 7.2\% loss on SmolLM2-1.7B) compared to per-block methods like LPBQ~\cite{aimet_lpbq_2026}. To reconcile this trade-off, \sys adopts a hardware-friendly coarse-grained configuration to accelerate inference while preserving model performance.

\subsection{Quantization}
\label{sec:bkg:r-qat}
Quantization~\cite{gholami2021surveyquantizationmethodsefficient, frantar2023gptqaccurateposttrainingquantization} reduces memory footprint and latency by lowering numerical precision. For a full-precision tensor $X_{FP}$, uniform affine quantization is defined as $X_Q = \alpha ( \lfloor X_{FP}/\alpha + \beta \rceil - \beta )$, where $\alpha$ is the scale and $\beta$ is the zero-point. Symmetric schemes set $\alpha = \max(|X_{FP}|)/(2^{N-1}-1)$ and $\beta = 0$, while asymmetric schemes define $\alpha = (\max(X_{FP})-\min(X_{FP}))/(2^N-1)$ and $\beta = \text{round}(-\min(X_{FP})/\alpha)$, where $N$ is the bit-width. However, LLMs present a challenge due to activation distributions with significant outliers~\cite{dettmers2022llmint88bitmatrixmultiplication, xiao2024smoothquantaccurateefficientposttraining}, which lead to a marked drop in performance.

\textbf{Rotation-based Quantization} 
Prior works\cite{ashkboos2024quarotoutlierfree4bitinference, liu2025spinquantllmquantizationlearned} have shown that Hadamard matrices ($H \in \{+1, -1\}^{n \times n}$) are particularly effective in mitigating outliers, as they can redistribute extreme values across all channels, thereby suppressing their impact. Owing to the orthogonality ($H^\top H = I$), the following equivalence always holds: $Y = X W^\top = (X H)(H^\top W^\top)$. This property allows activations and weights to be rotated without altering the end-to-end inference output of the model.

Prior work~\cite{liu2025spinquantllmquantizationlearned} has explored four types of Hadamard matrices for activations (Figure~\ref{fig:framework}): $R_1$ (shared across layers), $R_2$ (applied to Value), $R_3$ (applied to Query and Key), and $R_4$ (applied to down-projection). Among these, $R_1$ and $R_2$ can be fused into the weights ($WH$) offline, eliminating runtime overhead. It further improves $R_1$ and $R_2$ using Cayley optimization to discover superior rotation matrices while maintaining orthogonality. In contrast, $R_3$ and $R_4$ require online floating-point matrix multiplications during inference, introducing approximately 2$\times$ computational overhead.

\section{Motivation}
\label{sec: init motivation}

\noindent \textbf{Joint optimization of static parameters and rotation matrices.}
Existing optimization-based methods typically rely on dynamic quantization to simulate errors. However, as shown in \cref{fig:a}, the dynamic activation scale for $W_o$ of the $21^{th}$ layer fluctuates significantly during the optimization (e.g., a variance of 0.58). This instability indicates that dynamic simulation fails to capture the static quantization error. Consequently, directly converting the optimized model to static quantization for on-NPU inference leads to notable accuracy degradation (e.g., 15.61\% for the SmolLM2-1.7B-Instruct model). In contrast, fixing quantization parameters before optimization prevents such mismatch but severely hinders convergence, resulting in suboptimal performance as shown in \cref{fig:b} (e.g., loss variance of 0.3). This is because the optimal range of quantization parameters continuously shifts as rotation matrices are optimized. Therefore, the strong coupling between rotation matrices and quantization parameters implies that static calibration prior to joint optimization is inadequate.

\noindent \textbf{Impact of initialization on convergence.} 
Our experiments further reveal that quantization parameters initialization is critical for optimization. As shown in \cref{fig:b}, poor initialization can significantly slow convergence, leading to inferior results (e.g., final loss variance of 0.8) or even divergence.

\noindent \textbf{Selectivity of learnable quantization parameters.} 
A quantized LLM encompasses a number of static quantization parameters. Counterintuitively, \cref{fig:c} shows that increasing the number of learnable quantization parameters does not improve performance. Specifically, jointly optimizing the input activations, output activations, weight of linear layers and KV tensors results in the lowest accuracy. In contrast, excluding output activations improves performance, while optimizing only the input activation and weight quantization parameters of linear layers yields the best performance.

\noindent \textbf{Latency-accuracy trade-off in rotation-based quantization.}
Recent rotation-based quantization methods~\cite{ashkboos2024quarotoutlierfree4bitinference, liu2025spinquantllmquantizationlearned} effectively mitigate outliers by applying Hadamard matrices to activations. While offline-fused rotations ($R_1, R_2$) incur no latency, online rotations ($R_3, R_4$) require floating-point matrix multiplications during inference, introducing approximately $2\times$ computational overhead. To improve NPU inference efficiency, \sys completely eliminates these online rotations. However, this removal causes the input activations of down\_proj to exhibit severe outliers, making them significantly more difficult to quantize and leading to substantial performance degradation.

\textbf{Insights: Motivated by these findings, Quant.npu jointly optimizes static quantization parameters and rotation matrices through three key strategies: (i) rotation-and-bit-width-aware initialization (\cref{sec:init}) to stabilize convergence; (ii) distribution-aware selective optimization (\cref{sec:train}) to improve optimization effectiveness and manage complexity; (iii) adaptive mixed-precision (\cref{sec:adaptive mixed precision}) to preserve accuracy while eliminating high-latency online rotations.}

\section{Method}

\subsection{Overall Framework}
In this section, we analyze the distribution differences between rotated and unrotated activations and introduce our integer-only fully static quantization framework, \sys. As illustrated in ~\cref{fig:framework}, it consists of three key components: 1) a rotation-and-bit-width-aware initialization strategy (~\cref{sec:init}), which secures an optimal initial state and stabilizes subsequent optimization; 2) a distribution-aware selective optimization (~\cref{sec:train}), which decouples the quantization into a two-stage quantization pipeline, isolating sensitive heavy-tailed tensors from joint optimization, thereby reducing resource overhead and improving model performance; and 3) an outlier-aware adaptive mixed-precision strategy (~\cref{sec:adaptive mixed precision}), which mitigates outlier effects introduced by removing online rotation matrices (R4) with minimal computational cost while preserving model accuracy.

\subsection{Rotation-and-bit-width-aware Initialization}
\label{sec:init}

\noindent \textbf{Analysis of Scale Initialization.} 
We first analyze two representative methods: \textit{Mean-based}~\cite{bhalgat2020lsqimprovinglowbitquantization} and \textit{Max-Min}. The former determines quantization parameters from distribution statistics (mean and variance), while the latter relies on the full value range: $s_{init} = max(|\mu - 3\sigma|, |\mu + 3\sigma|) / 2^{b-1}$ and $s_{init} = (X_{max} - X_{min}) / (2^{b}-1)$, respectively. Here, $b$ denotes the bit-width, and $\mu$ and $\sigma$ are the mean and standard deviation, and $X_{max}$/$X_{min}$ denote the maximum/minimum values of the tensor. A more comprehensive analysis of additional initialization methods is provided in the Appendix~\ref{Appendix: act init}.

As indicated by the equations, the \textit{Max-Min} method captures a broad dynamic range, ensuring uniform coverage of all data points. However, this often results in coarser quantization granularity for the dense regions of the data. In contrast, the \textit{Mean-based} method provides higher quantization precision for the central distribution, effectively clipping extreme values (e.g., outliers). Consequently, \textit{Max-Min} initialization is preferable for higher bit-widths (e.g., 8-bit or 16-bit), where the quantization error remains negligible even when the numerical distribution spans a wide range. Conversely, when the data approximately follows a Gaussian distribution, the \textit{Mean-based} method yields better performance under low-bit quantization by minimizing error in the most information-rich regions.

\noindent \textbf{Activation Distribution Analysis.} 
Furthermore, we categorize activations into two types based on rotation: \textit{unrotated} and \textit{rotated}, as illustrated in ~\cref{fig:R_UR_8}. Rotated activations exhibit a Gaussian distribution, whereas unrotated activations typically follow a heavy-tailed distribution. This distinction aligns well with the characteristics of the initialization methods discussed above. Therefore, for rotated activations (e.g., input activation of $W_q$, $W_o$, $W_{up}$), we employ \textit{Mean-based} initialization. Conversely, unrotated activations require \textit{Max-Min} initialization to accommodate their wide range.

\noindent \textbf{Initialization Rule.} 
Based on the above analysis, we summarize our choice of initialization methods as follows: 
(1) For rotated activations, we adopt \textit{Mean-based} initialization to support lower-bit quantization.
(2) For unrotated activations, we use \textit{Max-Min} initialization with higher bit-widths (e.g., 8-bit or 16-bit) to accommodate their wide range, thereby ensuring model accuracy.
Detailed discussion and theoretical analysis of this initialization rule are provided in Appendix~\ref{Appendix: rule} and Appendix~\ref{Appendix: prove}.

\begin{figure*}[t]
  \centering
  \includegraphics[width=\linewidth]{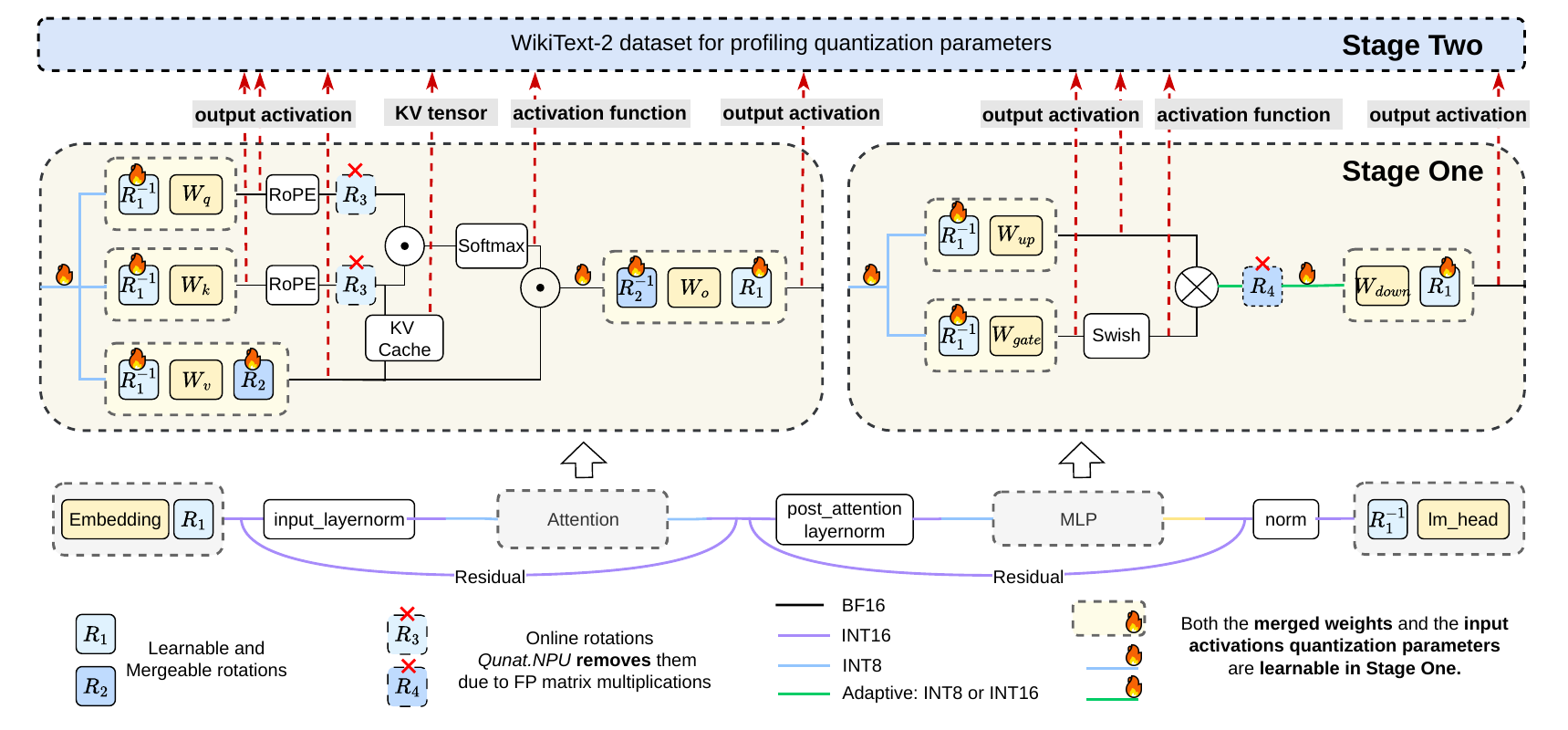}
  % \vspace{-10pt}
  \caption{
  The Overview of \sys. The quantization is divided into two stages. In the first stage, we optimize the quantization parameters for the "hot" merged weights and activations while keeping the other components at BF16 precision. In the second stage, we directly apply static calibration to the remaining activations and weights. In the diagram, black lines represent BF16, purple represents INT16, blue represents INT8, and green indicates a mixed-precision scheme of INT8 and INT16.
  }
  \label{fig:framework}
\end{figure*}

\subsection{Distribution-Aware Selective Optimization}
\label{sec:train}
In this work, we categorize quantization parameters into two types: (1) those associated with rotated distributions, such as the input activations and weights of linear layers; and (2) those associated with non-rotated distributions, including the output activations of linear layers and key/value tensors.

The fundamental distinction between these categories lies in their distribution smoothness and quantization sensitivity. Rotated distributions are significantly smoother and approximate a Gaussian profile (see ~\cref{fig:R_1}--~\cref{fig:R_4}), allowing their quantization parameters to be stably updated during the optimization. In contrast, non-rotated distributions exhibit pronounced heavy-tailed behavior (see ~\cref{fig:UR_1}--~\cref{fig:UR_4}), with values concentrated near zero. These distributions are inherently more sensitive to quantization, making their parameters difficult to optimize and prone to instability. A more detailed theoretical analysis of the optimization instability is provided in Appendix~\ref{Appendix: Gradient}. As illustrated in ~\cref{fig:c}, the convergence time is substantially increased without performance improvement. Furthermore, our initialization analysis in ~\cref{sec:init} indicates that non-rotated tensors require 8-bit or higher precision. Given this sufficient representational capacity, directly applying static calibration without further optimization to these components incurs negligible accuracy degradation.

Leveraging these insights, we propose a selective optimization pipeline that decouples the quantization process into two stages: Gradient-based Optimization and Static Calibration. This design significantly reduces optimization complexity and enhances stability without compromising model accuracy.

\subsubsection{Stage One: Gradient-based Optimization}
In the first stage, we jointly optimize the quantization parameters of input activations and weights for all linears together with the rotation matrices. Specifically, following prior work~\cite{liu2025spinquantllmquantizationlearned}, the input activations, output activations, and weights of the lm\_head linear are not quantized in this stage and remain in floating-point precision. Following ~\cite{esser2020learnedstepsizequantization, bhalgat2020lsqimprovinglowbitquantization}, we derive gradients for scale and zero-point parameters, apply gradient scaling, and introduce a local error loss to improve optimization stability.

\noindent \textbf{Learnable Scale and Zero-Point.}
We adopt the straight-through estimator (STE) ~\cite{bengio2013estimatingpropagatinggradientsstochastic} to approximate gradients through the rounding operator, to achieve differentiable optimization of quantitative parameters. The gradient of the scale parameter $s$ is approximated as:
\begin{equation}
\frac{\partial \hat{x}}{\partial s} = \frac{\partial \bar{x}}{\partial s}\, s + \bar{x} - z_p \approx
\begin{cases}
-\left( \dfrac{x}{s} + z_p \right) + \left\lfloor \dfrac{x}{s} + z_p \right\rceil, & \text{if } q_{\min} < \dfrac{x}{s} + z_p < q_{\max}, \\[6pt]
q_{\min} - z_p \;\; \text{or} \;\; q_{\max} - z_p, & \text{otherwise}.
\end{cases}
\label{eq:d_x_ds}
\end{equation}
Similarly, the gradient of the zero-point $z_p$ is:
\begin{equation}
\frac{\partial \hat{x}}{\partial z_p} = \left( \frac{\partial \bar{x}}{\partial z_p} - 1 \right) s \approx
\begin{cases}
0, & \text{if } q_{\min} < \dfrac{x}{s} + z_p < q_{\max}, \\[6pt]
-s, & \text{otherwise}.
\end{cases}
\end{equation}

\noindent \textbf{Gradient Scaling.}
During optimization, both rotation matrices and quantization parameters are updated jointly. However, the gradient magnitude of quantization parameters increases with the number of elements and bit-width, which leads to the gradient of the quantization parameter being much greater than that of the rotation matrix. To stabilize optimization, we apply the gradient scaling factor $g = 1 / \sqrt{N \cdot Q_{max}}$ proposed in the prior work~\cite{esser2020learnedstepsizequantization}, where $N$ is the number of elements sharing the same quantization parameter, and $Q_{max}$ is the maximum representable integer value.

\noindent \textbf{Local Quantization Error Loss}
To further stabilize early-stage optimization and accelerate convergence, we introduce a local quantization error loss for activation quantizers:
\begin{align}
\min_{s, z_p} \; \left\| \mathcal{DQ}(\mathcal{Q}(\mathbf{x})) - \mathbf{x} \right\|_2^2
\end{align}
where $\mathcal{Q}(\cdot)$and $\mathcal{DQ}(\cdot)$ respectively represent quantization and dequantization operators. This loss directly constrains the reconstruction error between original and dequantized activations, enabling quantization parameters to converge to favorable values early in the optimization process.

\subsubsection{Stage Two: Static Calibration}
After the first stage (Gradient-based Optimization), we directly apply static calibration to the quantization parameters for the remaining operators, without further optimization. These include the output activations of linear layers, key/value tensors, and non-linear activation functions such as SiLU.

\subsection{Adaptive Mixed-Precision Quantization}
\label{sec:adaptive mixed precision}
As discussed in ~\cref{sec:bkg:r-qat}, we eliminate online rotation matrices to avoid floating-point overhead during inference. Consequently, down\_proj input activations are no longer smoothed by rotation. Following the principles in ~\cref{sec:init}, these activations require a minimum precision of 8-bit.

However, prior work~\cite{xu2024fastondevicellminference} indicates that down\_proj inputs exhibit severe outliers, as illustrated in Appendix~\ref{Appendix: down_proj outlier}, causing significant accuracy degradation even at 8-bit quantization (see~\cref{fig:adaptive}). Further analysis reveals that quantization sensitivity varies significantly across linear layers~\cite{xu2024fastondevicellminference}. We observe that selectively promoting a small fraction of highly sensitive activations across all down\_proj layers from 8-bit to 16-bit effectively recovers model accuracy. Leveraging this insight, we propose an adaptive mixed-precision strategy that assigns bit-widths based on a quantization sensitivity metric.

\noindent \textbf{Quantization Sensitivity Metric.}
We identify two distinct activation distribution patterns that lead to quantization-induced accuracy degradation:
(1) The activation distribution has a large dynamic range that exceeds the representational capacity of the static quantizer.
(2) The activation values are mostly clustered near zero, so quantization preserves these small values but loses the important information carried by a few outliers. Both cases are characterized by the amplification of quantization error due to outliers. We therefore define a quantization sensitivity metric based on relative quantization error:
\begin{align}
\text{ratio} = \frac{1}{N} \sum_{i=1}^{N} 
\frac{\left| \text{dequantized}_i - \text{original\_activation}_i \right|}
     {\left| \text{original\_activation}_i \right| + 10^{-8}}
\end{align}
where \text{original\_activation} denotes the original activation tensor, \text{dequantized} is the corresponding dequantized tensor, and $N$ is the number of elements. In practice, we quantize activations using low-bit \textit{Max-Min} initialization and compute this metric. A ratio close to 1 indicates severe outlier-induced error and insufficient bit-width, whereas a ratio near 0 indicates the quantization is nearly lossless.

\noindent \textbf{Adaptive Mixed-Precision (8-/16-bit).}
The adaptive strategy targets down\_proj input activations. We initially apply 8-bit quantization with \textit{Max–Min} initialization to these, compute sensitivity ratios, and selectively increase the bit-width of highly sensitive activations to 16-bit. As shown in ~\cref{sec: adaptive experiments}, adjusting only the top 10\% of activations by the sensitivity metric (e.g., for a model with 30 layers, only 3 layers are adjusted) is sufficient to achieve performance comparable to full 16-bit quantization.

\begin{table*}[!t]
\centering
\small
\setlength{\tabcolsep}{4pt}
\renewcommand{\arraystretch}{1.15}
\caption{\textbf{Evaluation on Llama-3.2-3B-Instruct}: Perplexity(PPL) and Zero-shot QA task accuracy results of 4-bit weight or 8-bit weight quantized models.}
\label{tab:zero_shot_qa}
\begin{adjustbox}{max width=0.97\linewidth}
\begin{tabular}{c | c c c | c | c c c c c c c} 
\toprule
Model & Method & W-A-KV & W. Gran. & PPL(C4) $\downarrow$ & PIQA & Winogrande & HellaSwag & ARC-E & ARC-C & LAMBADA & Avg $\uparrow$ \\
\midrule

% ===================== Llama-3.2-3B-Instruct =====================
\multirow{14}{*}{\makecell{Llama-3.2\\-3B-it}}
& FP32       & --    & -- & 16.85 & 0.7541 & 0.6748 & 0.7002 & 0.5593 & 0.4198 & 0.5696 & 0.613 \\
& executorch & 4-16-8 & per-block & 18.19 & 0.7361 & 0.6772 & 0.6878 & 0.5328 & 0.3942 & 0.5329 & 0.5935 \\
& executorch  & 8-8-8 & per-channel & 17.76 & 0.7372 & 0.6638 & 0.694 & \textbf{0.5518} & \textbf{0.4266} & 0.5527 & 0.6044 \\
& QuaRot  & 8-8-8 & per-channel & 17.06 & 0.7448 & 0.6654 & 0.6958 & 0.5467 & 0.4155 & 0.5639 & 0.6054 \\
& SpinQuant  & 8-8-8 & per-channel & 17.18 & 0.7405 & 0.6677 & 0.6936 & 0.5501 & 0.4113 & 0.5577 & 0.6035 \\
& \cellcolor{lightred}Quant.npu & \cellcolor{lightred}8-8-8 & \cellcolor{lightred}per-channel & \cellcolor{lightred}\textbf{16.42} & \cellcolor{lightred}\textbf{0.7541} & \cellcolor{lightred}\textbf{0.6827} & \cellcolor{lightred}\textbf{0.7003} & \cellcolor{lightred}0.5446 & \cellcolor{lightred}0.4104 & \cellcolor{lightred}\textbf{0.6066} & \cellcolor{lightred}\textbf{0.6165} \\
& executorch  & 4-8-8 & per-channel & 26.39 & 0.7040 & 0.6046 & 0.628 & 0.4798 & 0.3635 & 0.3835 & 0.5272 \\
& QuaRot  & 4-8-8 & per-channel & 30.06 & 0.6708 & 0.5691 & 0.5755 & 0.4609 & 0.3234 & 0.2600 & 0.4766 \\
& SpinQuant  & 4-8-8 & per-channel & 28.78 & 0.6393 & 0.5651 & 0.5620 & 0.4179 & 0.2978 & 0.2915 & 0.4623 \\
& \cellcolor{lightred}Quant.npu & \cellcolor{lightred}4-8-8 & \cellcolor{lightred}per-channel & \cellcolor{lightred}\textbf{19.16} & \cellcolor{lightred}\textbf{0.7378} & \cellcolor{lightred}\textbf{0.6496} & \cellcolor{lightred}\textbf{0.6777} & \cellcolor{lightred}\textbf{0.5105} & \cellcolor{lightred}\textbf{0.3891} & \cellcolor{lightred}\textbf{0.5317} & \cellcolor{lightred}\textbf{0.5827} \\
& executorch  & 4-4-8 & per-channel & 6950.81 & 0.5196 & 0.5138 & 0.2493 & 0.2601 & 0.2474 & 0.0 & 0.2984 \\
& QuaRot  & 4-4-8 & per-channel & 481.82 & 0.5196 & 0.5091 & 0.2768 & 0.2689 & 0.2517 & 0.0014 & 0.3046 \\
& SpinQuant  & 4-4-8 & per-channel & 528.97 & 0.5103 & 0.4878 & 0.2731 & 0.2614 & 0.2295 & 0.0019 & 0.2940 \\
& \cellcolor{lightred}Quant.npu & \cellcolor{lightred}4-4-8 & \cellcolor{lightred}per-channel & \cellcolor{lightred}\textbf{21.76} & \cellcolor{lightred}\textbf{0.7209} & \cellcolor{lightred}\textbf{0.6219} & \cellcolor{lightred}\textbf{0.6420} & \cellcolor{lightred}\textbf{0.4916} & \cellcolor{lightred}\textbf{0.3695} & \cellcolor{lightred}\textbf{0.4689} &
\cellcolor{lightred}\textbf{0.5525} \\
\bottomrule
\end{tabular}
\end{adjustbox}
\end{table*}

\section{Experiments}
\label{Experiments}
\noindent \textbf{Evaluation models and datasets.} 
Our evaluation is conducted on LLaMA-3.2-3B-Instruct~\cite{grattafiori2024llama3herdmodels}. Following prior work~\cite{ashkboos2024quarotoutlierfree4bitinference}, we report the perplexity (PPL) of language generation tasks on the C4 datasets~\cite{raffel2023exploringlimitstransferlearning}. In addition, we evaluate zero-shot commonsense reasoning performance on six standard benchmarks: PIQA~\cite{bisk2019piqareasoningphysicalcommonsense}, WinoGrande~\cite{sakaguchi2019winograndeadversarialwinogradschema}, HellaSwag~\cite{zellers2019hellaswagmachinereallyfinish}, ARC-Easy, ARC-Challenge~\cite{clark2018thinksolvedquestionanswering} and LAMBADA~\cite{paperno2016lambadadatasetwordprediction}. All evaluations are developed with reference to the lm-eval-harness framework.

\noindent \textbf{Baselines}
We compare against three state-of-the-art frameworks: (i) ExecuTorch~\cite{pytorch_executorch_2026}, an edge runtime supporting Hexagon NPU inference with per-block weight PTQ; (ii) QuaRot~\cite{ashkboos2024quarotoutlierfree4bitinference}, a rotation-based PTQ method with random Hadamard matrices; and (iii) SpinQuant~\cite{liu2025spinquantllmquantizationlearned}, a rotation-based PTQ method with learnable rotation matrices, utilizing dynamic quantization during optimization.

\noindent \textbf{Implementation Details.}
The optimization stage in experiments employs 4 NVIDIA A40 GPUs (48 GB). We optimize on WikiText-2 (bfloat16, seq\_len 2048) for 512 steps (7 epochs) using a batch size of 2 and 2 gradient accumulation steps. To stabilize optimization, the local quantization error loss is applied for the first 128 steps (~\cref{sec:train}). We use SGD with cosine decay, setting learning rates to 0.1 for rotation matrices and 0.01 for static quantization parameters.

\noindent \textbf{Quantization.}
\sys adopts pre-channel symmetric weight quantization and per-tensor symmetric activation quantization. For bit-width assignment, we only vary the quantization precision of input activations and weights of linear layers, while keeping other tensors at fixed precision. Specifically, KV cache and output activations of linear layers are fixed at 8-bit precision, whereas remaining operations (e.g., SiLU) are fixed at 16-bit precision to preserve model accuracy. In particular, for the down\_proj layers, the activation bit-width is assigned using the adaptive mixed-precision strategy proposed in ~\cref{sec:adaptive mixed precision}. All static quantization parameters are initialized based on activation characteristics following in ~\cref{sec:init}. Calibration utilizes 16 samples from WikiText-2~\cite{merity2016pointersentinelmixturemodels}.

\subsection{Main Results}
\label{Main Results}
\noindent \textbf{Setup.}
All baselines and our method Quant.npu are implemented on ExecuTorch (integer-only quantization for NPU inference) to ensure a fair comparison, modifying only the rotation matrices and quantizers. Specifically, ExecuTorch uses 4-bit block-wise weights (low-power blockwise quantization, LPBQ~\cite{aimet_lpbq_2026}) and 16-bit activation quantization. In contrast, another ExecuTorch variant, QuaRot, SpinQuant, and \sys share the same quantization configuration, including per-tensor activation quantization, per-channel weight quantization, and 16-bit activations for all down\_proj layers. The key difference is that QuaRot utilizes fixed Hadamard matrices, SpinQuant optimizes only rotation matrices, whereas \sys jointly optimizes rotations and static quantization parameters. Furthermore, \sys applies a two-stage quantization (~\cref{sec:train}) for operators like SiLU.

\noindent \textbf{Results: higher accuracy at iso-precision.}
\cref{tab:zero_shot_qa} reports the perplexity (PPL) on C4 and the accuracy on six zero-shot commonsense reasoning benchmarks under W8A8, W4A8, and W4A4 settings. Across all major benchmarks, \sys outperforms state-of-the-art quantization approaches while maintaining equivalent inference efficiency. Notably, under W8A8 quantization, Llama-3.2-3B-Instruct achieves zero-shot accuracy comparable to, or even slightly better than, its FP32 baseline. The advantage of \sys becomes even more pronounced under the challenging W4A8 setting. For Llama-3.2-3B-Instruct, \sys-W4A8 reduces PPL from 28.78 to 19.16 and boosts zero-shot accuracy from 46.23\% to 58.27\%, representing a 12.04\% improvement. Furthermore, \sys-W4A4 maintains a relatively high accuracy, even in scenarios where the baseline almost entirely fails, exhibiting only about a 6\% drop in accuracy compared to the FP32 baseline. Extensive supplementary experiments are detailed in Appendix~\ref{Appendix: Additional Experimental}. We validate the broad applicability of \sys across newer architectures (Qwen3-1.7B~\cite{yang2025qwen3}) and larger model (Llama3-8B~\cite{grattafiori2024llama3herdmodels}) , while also verifying its ability to preserve instruction-following capabilities. Additional evaluation on more widely used models, comparisons with MobileQuant~\cite{tan2024mobilequant}, and profiling of memory and energy consumption on the newer SM8750 NPU are provided to further confirm our method's superiority.

\subsection{End-to-end Latency}
\label{End-to-end Latency}
\noindent \textbf{Setup.}
We evaluate inference latency on the SM8650 NPU using HellaSwag~\cite{zellers2019hellaswagmachinereallyfinish}, Persona-Chat~\cite{jandaghi2023faithfulpersonabasedconversationaldataset}, and DroidTask~\cite{wen2024autodroidllmpoweredtaskautomation}, spanning varied token lengths to show system performance under varying loads.

\textbf{Results: faster inference with negligible accuracy loss.} \sys-W4A8 yields a up to 15.1\% speedup compared to ExecuTorch-W4A16 with a negligible accuracy drop of 1\% on Llama-3.2-3B-Instruct (averaging a 2.58\% drop and 1.23 PPL increase across all tasks and all models). These results demonstrate that the joint optimization of rotation matrices and static quantization parameters enables \sys to adopt lower-bit and more coarse-grained quantization than ExecuTorch-W4A16, thereby improving inference efficiency without significantly compromising performance.

\begin{figure}[t]
  \centering

  \begin{minipage}[t]{0.48\linewidth}
    \centering
    \includegraphics[width=\linewidth]{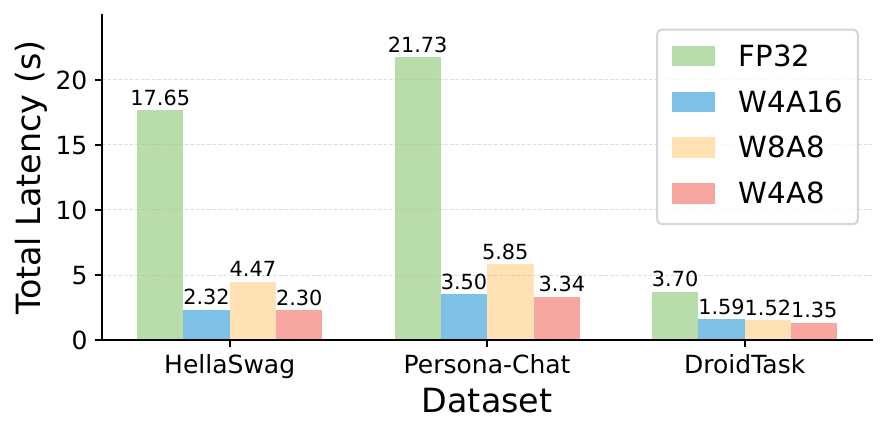} 
    \caption{End-to-end latency results, where blue bars represent the per-block weight quantization in ExecuTorch; the yellow and red bars indicate per-channel weight quantization in \sys.}
    \label{fig:latency}
  \end{minipage}
  \hfill
  \begin{minipage}[t]{0.48\linewidth}
    \centering
    \includegraphics[width=\linewidth]{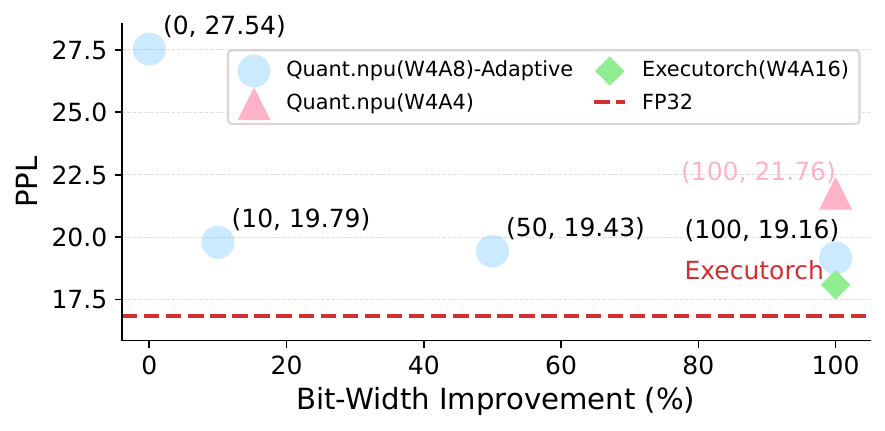}
    \caption{The trend of perplexity (PPL) with respect to the percentage of linear-layer activations quantized to higher bit-widths for the Llama-3.2-3B-Instruct model.}
    \label{fig:adaptive}
  \end{minipage}

\end{figure}

\subsection{Adaptive strategy}
\label{sec: adaptive experiments}
\noindent \textbf{Setup.} 
We employ an adaptive mixed-precision strategy prior to optimization, where activations of down\_proj layers are selectively promoted to 16-bit precision to better preserve model accuracy. Consistent with ~\cref{Main Results}, all evaluations fully utilize a pure integer-arithmetic pipeline. We further measure perplexity (PPL) on the C4 dataset to validate the effectiveness of the proposed strategy.

\noindent \textbf{Results.}
\cref{fig:adaptive} illustrates the impact of mixed-precision ratios of all down\_proj activations on perplexity (PPL). We observe that increasing the precision of only a small fraction of down\_proj activations yields substantial accuracy gains. For instance, promoting just 10\% of activations to 16-bit precision results in a PPL increase of only 0.63 compared to the full 16-bit baseline, while reducing PPL by 7.75 relative to the full 8-bit setting. Based on this key insight, \sys further reduces the use of high-bit activations, thereby significantly improving inference efficiency in practice with only marginal accuracy degradation compared to well-established ExecuTorch-W4A16 baseline.

\begin{table*}[!t]
\centering
\small
\setlength{\tabcolsep}{6pt}
\renewcommand{\arraystretch}{1.15}
\caption{\textbf{Ablation study of Quant.npu’s main components on Llama3.2-3B-Instruct}:  Perplexity (PPL) and zero-shot QA task accuracy results under different component settings. Here, "D/S" indicates whether dynamic or static quantization was used during optimization; "IM" refers to the initialization methodology; "Sel-Opt" denotes the application of our selective optimization strategy; "Adaptive" indicates whether a mixed-precision strategy was applied, with the value in parentheses representing the percentage of positions where the quantization bit-width was increased. The FP32 model does not involve any quantization-related components (D/S, Optim, IM, Sel-Opt, Adaptive). Thus, these fields are marked as "--", and this row represents the performance of the FP32 model.}
\label{tab:ablation}
\begin{adjustbox}{max width=0.97\textwidth}
\begin{tabular}{c | c | c c c c c | c c c c c} 
\toprule
Model & Method & D/S & Optim & \makecell{IM \\ (Sec. 4.2)} & \makecell{Sel-Opt \\ (Sec. 4.3)} & \makecell{Adaptive (\%) \\ (Sec. 4.4)} & PPL(C4) $\downarrow$ & Winogrande $\uparrow$ & ARC-C $\uparrow$ & LAMBADA $\uparrow$ & Avg $\uparrow$\\
\midrule
\multirow{9}{*}{\makecell{Llama-3.2\\-3B-it}} 
& FP32 & -- & -- & -- & -- & -- & 22.64 & 0.6093 & 0.3993 & 0.4731 & 0.4939 \\
& SpinQuant & D & R &  &  &  & 1187.69 & 0.4917 & 0.2619 & 0.0004 & 0.2513 \\
& -- & S & R &  &  &  & 150.58 & 0.5036 & 0.2184 & 0.0704 & 0.2641 \\
& -- & S & R and Scale (Joint) &  &  &  & 69.65 & 0.5130 & 0.2662 & 0.1114 & 0.2969 \\
& -- & S & R and Scale (Joint) & \checkmark &  &  & 36.42 & 0.5659 & 0.2910 & 0.2820 & 0.3796 \\
& -- & S & R and Scale (Joint) & \checkmark & \checkmark &  & 30.89 & 0.5809 & 0.3114 & 0.3122 & 0.4015 \\
& \cellcolor{lightred}Ours & \cellcolor{lightred}S & \cellcolor{lightred}R and Scale (Joint) & \cellcolor{lightred}\checkmark & \cellcolor{lightred}\checkmark & \cellcolor{lightred}\checkmark(10\%) & \cellcolor{lightred}22.09 & \cellcolor{lightred}0.6062 & \cellcolor{lightred}0.3703 & \cellcolor{lightred}0.4434 & \cellcolor{lightred}0.4733 \\
& -- & S & R and Scale (Joint) & \checkmark & \checkmark & \checkmark(100\%) & 21.56 & 0.6219 & 0.3456 & 0.4681 & 0.4785 \\
& \makecell{High Precision \\ Alone} & S & R and Scale (Joint) &  &  & \checkmark(100\%) & 31.16 & 0.5651 & 0.3055 & 0.3214 & 0.3973 \\
\bottomrule
\end{tabular}
\end{adjustbox}
\end{table*}

\subsection{Ablation Study}
\label{sec: adaptive experiments}
\noindent \textbf{Setup.} 
To better demonstrate the effectiveness of our proposed method, we conduct an ablation study on the key components of \sys. All experiments are performed on LlaMA-3.2-3B-Instruct. To clearly present the results, we quantize only the weights, input activations, and output activations of the linear layers, while other operators such as SiLU remain in full precision.

\noindent \textbf{Results.}
As summarized in Table~\ref{tab:ablation}, starting from a naive static quantization baseline, we progressively integrate each component of our method. In the baseline setting, where both rotation matrices and quantization parameters are directly optimized under static quantization constraints, the model suffers significant degradation, yielding a perplexity (PPL) of 69.65 and an average accuracy of 0.2969. Introducing our proposed initialization strategy significantly stabilizes optimization process and improves optimization quality, reducing PPL to 36.42 and improving average accuracy to 0.3796. Building on this, the proposed distribution-aware selective optimization, implemented through a two-stage quantization pipeline, further improves model performance, achieving a PPL of 30.89 and an average accuracy of 0.4015. Finally, by incorporating the adaptive mixed-precision strategy, only 10\% of the down\_proj activations (3 out of 28 layers) are promoted to 16-bit precision, while the remaining activations are kept at 8-bit, achieving a PPL of 22.09 and an average accuracy of 0.4733. Notably, the resulting performance remains nearly identical to that of both the FP32 (Avg: 0.4939) and full 16-bit (Avg: 0.4785) settings. These results demonstrate that each proposed component contributes to performance improvement, and that \sys effectively maintains performance close to the FP32 model, achieving a favorable trade-off between accuracy and inference efficiency.

\section{Conclusions}
We have proposed \sys, an NPU-friendly fully static quantization framework for on-device LLMs that integrates learnable quantization parameters and rotation matrices to enable efficient low-bit inference. By incorporating a rotation-and-bit-width-aware initialization strategy together with a decoupled two-stage selective optimization pipeline, \sys achieves competitive accuracy while significantly reducing inference latency. Furthermore, \sys demonstrates strong generalization across different model architectures and scales, highlighting its scalability in practical deployment scenarios. Experiments confirm that \sys achieves comparable accuracy to state-of-the-art methods while delivering up to a 15.1\% reduction in latency on mobile NPUs.

% \begin{ack}
% Use unnumbered first level headings for the acknowledgments. All acknowledgments

% Do {\bf not} include this section in the anonymized submission, only in the final paper. You can use the \texttt{ack} environment provided in the style file to automatically hide this section in the anonymized submission.
% \end{ack}

\bibliographystyle{plain}
\bibliography{reference}

\begin{thebibliography}{10}

\bibitem{allal2025smollm2smolgoesbig}
Loubna~Ben Allal, Anton Lozhkov, Elie Bakouch, Gabriel~Martín Blázquez, Guilherme Penedo, Lewis Tunstall, Andrés Marafioti, Hynek Kydlíček, Agustín~Piqueres Lajarín, Vaibhav Srivastav, and et~al.
\newblock Smollm2: When smol goes big -- data-centric training of a small language model, 2025.

\bibitem{ashkboos2024quarotoutlierfree4bitinference}
Saleh Ashkboos, Amirkeivan Mohtashami, Maximilian~L. Croci, Bo~Li, Pashmina Cameron, Martin Jaggi, Dan Alistarh, Torsten Hoefler, and James Hensman.
\newblock Quarot: Outlier-free 4-bit inference in rotated llms, 2024.

\bibitem{bengio2013estimatingpropagatinggradientsstochastic}
Yoshua Bengio, Nicholas Léonard, and Aaron Courville.
\newblock Estimating or propagating gradients through stochastic neurons for conditional computation, 2013.

\bibitem{bhalgat2020lsqimprovinglowbitquantization}
Yash Bhalgat, Jinwon Lee, Markus Nagel, Tijmen Blankevoort, and Nojun Kwak.
\newblock Lsq+: Improving low-bit quantization through learnable offsets and better initialization, 2020.

\bibitem{bisk2019piqareasoningphysicalcommonsense}
Yonatan Bisk, Rowan Zellers, Ronan~Le Bras, Jianfeng Gao, and Yejin Choi.
\newblock Piqa: Reasoning about physical commonsense in natural language, 2019.

\bibitem{bondarenko2024low}
Yelysei Bondarenko, Riccardo Del~Chiaro, and Markus Nagel.
\newblock Low-rank quantization-aware training for llms.
\newblock {\em arXiv preprint arXiv:2406.06385}, 2024.

\bibitem{chen2025efficientqat}
Mengzhao Chen, Wenqi Shao, Peng Xu, Jiahao Wang, Peng Gao, Kaipeng Zhang, and Ping Luo.
\newblock Efficientqat: Efficient quantization-aware training for large language models.
\newblock In {\em Proceedings of the 63rd Annual Meeting of the Association for Computational Linguistics (Volume 1: Long Papers)}, pages 10081--10100, 2025.

\bibitem{clark2018thinksolvedquestionanswering}
Peter Clark, Isaac Cowhey, Oren Etzioni, Tushar Khot, Ashish Sabharwal, Carissa Schoenick, and Oyvind Tafjord.
\newblock Think you have solved question answering? try arc, the ai2 reasoning challenge, 2018.

\bibitem{dettmers2022llmint88bitmatrixmultiplication}
Tim Dettmers, Mike Lewis, Younes Belkada, and Luke Zettlemoyer.
\newblock Llm.int8(): 8-bit matrix multiplication for transformers at scale, 2022.

\bibitem{dettmers2023qlora}
Tim Dettmers, Artidoro Pagnoni, Ari Holtzman, and Luke Zettlemoyer.
\newblock Qlora: Efficient finetuning of quantized llms.
\newblock {\em Advances in neural information processing systems}, 36:10088--10115, 2023.

\bibitem{du2024bitdistiller}
Dayou Du, Yijia Zhang, Shijie Cao, Jiaqi Guo, Ting Cao, Xiaowen Chu, and Ningyi Xu.
\newblock Bitdistiller: Unleashing the potential of sub-4-bit llms via self-distillation.
\newblock In {\em Proceedings of the 62nd Annual Meeting of the Association for Computational Linguistics (Volume 1: Long Papers)}, pages 102--116, 2024.

\bibitem{esser2020learnedstepsizequantization}
Steven~K. Esser, Jeffrey~L. McKinstry, Deepika Bablani, Rathinakumar Appuswamy, and Dharmendra~S. Modha.
\newblock Learned step size quantization, 2020.

\bibitem{frantar2023gptqaccurateposttrainingquantization}
Elias Frantar, Saleh Ashkboos, Torsten Hoefler, and Dan Alistarh.
\newblock Gptq: Accurate post-training quantization for generative pre-trained transformers, 2023.

\bibitem{gholami2021surveyquantizationmethodsefficient}
Amir Gholami, Sehoon Kim, Zhen Dong, Zhewei Yao, Michael~W. Mahoney, and Kurt Keutzer.
\newblock A survey of quantization methods for efficient neural network inference, 2021.

\bibitem{grattafiori2024llama3herdmodels}
Aaron Grattafiori, Abhimanyu Dubey, Abhinav Jauhri, Abhinav Pandey, Abhishek Kadian, Ahmad Al-Dahle, Aiesha Letman, Akhil Mathur, Alan Schelten, Alex Vaughan, and et~al.
\newblock The llama 3 herd of models, 2024.

\bibitem{gupta2015deep}
Suyog Gupta, Ankur Agrawal, Kailash Gopalakrishnan, and Pritish Narayanan.
\newblock Deep learning with limited numerical precision.
\newblock In {\em International conference on machine learning}, pages 1737--1746. PMLR, 2015.

\bibitem{jacob2018quantization}
Benoit Jacob, Skirmantas Kligys, Bo~Chen, Menglong Zhu, Matthew Tang, Andrew Howard, Hartwig Adam, and Dmitry Kalenichenko.
\newblock Quantization and training of neural networks for efficient integer-arithmetic-only inference.
\newblock In {\em Proceedings of the IEEE conference on computer vision and pattern recognition}, pages 2704--2713, 2018.

\bibitem{jandaghi2023faithfulpersonabasedconversationaldataset}
Pegah Jandaghi, XiangHai Sheng, Xinyi Bai, Jay Pujara, and Hakim Sidahmed.
\newblock Faithful persona-based conversational dataset generation with large language models, 2023.

\bibitem{lin2024awq}
Ji~Lin, Jiaming Tang, Haotian Tang, Shang Yang, Wei-Ming Chen, Wei-Chen Wang, Guangxuan Xiao, Xingyu Dang, Chuang Gan, and Song Han.
\newblock Awq: Activation-aware weight quantization for on-device llm compression and acceleration.
\newblock {\em Proceedings of machine learning and systems}, 6:87--100, 2024.

\bibitem{liu2024llm}
Zechun Liu, Barlas Oguz, Changsheng Zhao, Ernie Chang, Pierre Stock, Yashar Mehdad, Yangyang Shi, Raghuraman Krishnamoorthi, and Vikas Chandra.
\newblock Llm-qat: Data-free quantization aware training for large language models.
\newblock In {\em Findings of the Association for Computational Linguistics: ACL 2024}, pages 467--484, 2024.

\bibitem{liu2025spinquantllmquantizationlearned}
Zechun Liu, Changsheng Zhao, Igor Fedorov, Bilge Soran, Dhruv Choudhary, Raghuraman Krishnamoorthi, Vikas Chandra, Yuandong Tian, and Tijmen Blankevoort.
\newblock Spinquant: Llm quantization with learned rotations, 2025.

\bibitem{ma2023llmprunerstructuralpruninglarge}
Xinyin Ma, Gongfan Fang, and Xinchao Wang.
\newblock Llm-pruner: On the structural pruning of large language models, 2023.

\bibitem{merity2016pointersentinelmixturemodels}
Stephen Merity, Caiming Xiong, James Bradbury, and Richard Socher.
\newblock Pointer sentinel mixture models, 2016.

\bibitem{nagel2022overcoming}
Markus Nagel, Marios Fournarakis, Yelysei Bondarenko, and Tijmen Blankevoort.
\newblock Overcoming oscillations in quantization-aware training.
\newblock In {\em International Conference on Machine Learning}, pages 16318--16330. PMLR, 2022.

\bibitem{pamadi2025edge}
Vishesh~Narendra Pamadi and Pushpa Singh.
\newblock Edge ai vs cloud ai: A comparative study of performance latency and scalability.
\newblock {\em International Journal of Research in Modern Engineering \& Emerging Technology (IJRMEET)}, 13(3):13--35, 2025.

\bibitem{paperno2016lambadadatasetwordprediction}
Denis Paperno, Germán Kruszewski, Angeliki Lazaridou, Quan~Ngoc Pham, Raffaella Bernardi, Sandro Pezzelle, Marco Baroni, Gemma Boleda, and Raquel Fernández.
\newblock The lambada dataset: Word prediction requiring a broad discourse context, 2016.

\bibitem{pytorch_executorch_2026}
{PyTorch}.
\newblock {ExecuTorch: On‑Device AI Inference Powered by PyTorch}.
\newblock GitHub repository, 2026.
\newblock Version accessed Jan 2026.

\bibitem{qualcomm_applyencodings_2026}
{Qualcomm}.
\newblock Applyencodings, 2026.
\newblock Qualcomm Documentation. Accessed: 2026-01-29.

\bibitem{aimet_lpbq_2026}
{Qualcomm Innovation Center, Inc. (AIMET)}.
\newblock Low-power blockwise quantization (lpbq), 2026.
\newblock AIMET Documentation (Version 2.19.0). Accessed: 2026-01-29.

\bibitem{raffel2023exploringlimitstransferlearning}
Colin Raffel, Noam Shazeer, Adam Roberts, Katherine Lee, Sharan Narang, Michael Matena, Yanqi Zhou, Wei Li, and Peter~J. Liu.
\newblock Exploring the limits of transfer learning with a unified text-to-text transformer, 2023.

\bibitem{10592049}
Alejandro Rico, Satyaprakash Pareek, Javier Cabezas, David Clarke, Baris Ozgul, Francisco Barat, Yao Fu, Stephan Münz, Dylan Stuart, Patrick Schlangen, and et~al.
\newblock Amd xdna npu in ryzen ai processors.
\newblock {\em IEEE Micro}, 44(6):73--82, 2024.

\bibitem{sakaguchi2019winograndeadversarialwinogradschema}
Keisuke Sakaguchi, Ronan~Le Bras, Chandra Bhagavatula, and Yejin Choi.
\newblock Winogrande: An adversarial winograd schema challenge at scale, 2019.

\bibitem{shafee2025privacy}
Ahmed Shafee, SR~Hasan, and Tasneem~A Awaad.
\newblock Privacy and security vulnerabilities in edge intelligence: An analysis and countermeasures.
\newblock {\em Computers and Electrical Engineering}, 123:110146, 2025.

\bibitem{shao2024omniquantomnidirectionallycalibratedquantization}
Wenqi Shao, Mengzhao Chen, Zhaoyang Zhang, Peng Xu, Lirui Zhao, Zhiqian Li, Kaipeng Zhang, Peng Gao, Yu~Qiao, and Ping Luo.
\newblock Omniquant: Omnidirectionally calibrated quantization for large language models, 2024.

\bibitem{sun2024simpleeffectivepruningapproach}
Mingjie Sun, Zhuang Liu, Anna Bair, and J.~Zico Kolter.
\newblock A simple and effective pruning approach for large language models, 2024.

\bibitem{sun2025flatquantflatnessmattersllm}
Yuxuan Sun, Ruikang Liu, Haoli Bai, Han Bao, Kang Zhao, Yuening Li, Jiaxin Hu, Xianzhi Yu, Lu~Hou, Chun Yuan, Xin Jiang, Wulong Liu, and Jun Yao.
\newblock Flatquant: Flatness matters for llm quantization, 2025.

\bibitem{tan2024mobilequant}
Fuwen Tan, Royson Lee, {\L}ukasz Dudziak, Shell~Xu Hu, Sourav Bhattacharya, Timothy Hospedales, Georgios Tzimiropoulos, and Brais Martinez.
\newblock Mobilequant: Mobile-friendly quantization for on-device language models.
\newblock In {\em Findings of the Association for Computational Linguistics: EMNLP 2024}, pages 9761--9771, 2024.

\bibitem{wei2023outlier}
Xiuying Wei, Yunchen Zhang, Yuhang Li, Xiangguo Zhang, Ruihao Gong, Jinyang Guo, and Xianglong Liu.
\newblock Outlier suppression+: Accurate quantization of large language models by equivalent and optimal shifting and scaling.
\newblock {\em arXiv preprint arXiv:2304.09145}, 2023.

\bibitem{wen2024autodroidllmpoweredtaskautomation}
Hao Wen, Yuanchun Li, Guohong Liu, Shanhui Zhao, Tao Yu, Toby Jia-Jun Li, Shiqi Jiang, Yunhao Liu, Yaqin Zhang, and Yunxin Liu.
\newblock Autodroid: Llm-powered task automation in android, 2024.

\bibitem{xiao2024smoothquantaccurateefficientposttraining}
Guangxuan Xiao, Ji~Lin, Mickael Seznec, Hao Wu, Julien Demouth, and Song Han.
\newblock Smoothquant: Accurate and efficient post-training quantization for large language models, 2024.

\bibitem{xu2024fastondevicellminference}
Daliang Xu, Hao Zhang, Liming Yang, Ruiqi Liu, Gang Huang, Mengwei Xu, and Xuanzhe Liu.
\newblock Fast on-device llm inference with npus, 2024.

\bibitem{xu2023qa}
Yuhui Xu, Lingxi Xie, Xiaotao Gu, Xin Chen, Heng Chang, Hengheng Zhang, Zhengsu Chen, Xiaopeng Zhang, and Qi~Tian.
\newblock Qa-lora: Quantization-aware low-rank adaptation of large language models.
\newblock {\em arXiv preprint arXiv:2309.14717}, 2023.

\bibitem{yang2025qwen3}
An~Yang, Anfeng Li, Baosong Yang, Beichen Zhang, Binyuan Hui, Bo~Zheng, Bowen Yu, Chang Gao, Chengen Huang, Chenxu Lv, et~al.
\newblock Qwen3 technical report.
\newblock {\em arXiv preprint arXiv:2505.09388}, 2025.

\bibitem{qwen2025qwen25technicalreport}
An~Yang, Baosong Yang, Beichen Zhang, Binyuan Hui, Bo~Zheng, Bowen Yu, Chengyuan Li, Dayiheng Liu, Fei Huang, Haoran Wei, and et~al.
\newblock Qwen2.5 technical report, 2025.

\bibitem{zellers2019hellaswagmachinereallyfinish}
Rowan Zellers, Ari Holtzman, Yonatan Bisk, Ali Farhadi, and Yejin Choi.
\newblock Hellaswag: Can a machine really finish your sentence?, 2019.

\bibitem{zhang2024plug}
Yingtao Zhang, Haoli Bai, Haokun Lin, Jialin Zhao, Lu~Hou, and Carlo~Vittorio Cannistraci.
\newblock Plug-and-play: An efficient post-training pruning method for large language models.
\newblock 2024.

\bibitem{zhao2024atom}
Yilong Zhao, Chien-Yu Lin, Kan Zhu, Zihao Ye, Lequn Chen, Size Zheng, Luis Ceze, Arvind Krishnamurthy, Tianqi Chen, and Baris Kasikci.
\newblock Atom: Low-bit quantization for efficient and accurate llm serving.
\newblock {\em Proceedings of Machine Learning and Systems}, 6:196--209, 2024.

\end{thebibliography}

%%%%%%%%%%%%%%%%%%%%%%%%%%%%%%%%%%%%%%%%%%%%%%%%%%%%%%%%%%%%

\appendix

\section{Related Work}
Quantization is widely recognized as one of the most practical techniques for deploying large language models (LLMs) on resource-constrained platforms, as it effectively compresses model parameters and reduces computational and memory overhead. Compared to pruning-based approaches~\cite{ma2023llmprunerstructuralpruninglarge, sun2024simpleeffectivepruningapproach, zhang2024plug}, quantization does not alter the network architecture and often yields superior performance.

Early studies primarily focused on weight-only quantization, where model weights are represented with low-bit precision while activations remain in full precision. Representative methods such as GPTQ~\cite{frantar2023gptqaccurateposttrainingquantization} attain high accuracy at low bit-width by minimizing layer-wise reconstruction error. However, weight-only quantization suffers from an inherent limitation: during inference, full-precision activations prevent the use of fully integer matrix multiplication operators for efficient inference acceleration.

To overcome this bottleneck, recent research has shifted toward weight-activation quantization, which quantizes both weights and activations to enable fully integer computation. A key challenge in this setting is that activations in modern LLMs exhibit heavy-tailed distributions with extreme outliers, leading to substantial accuracy degradation under low-bit quantization. To address this issue, numerous methods have been proposed to mitigate the impact of activation outliers. Existing approaches can be broadly categorized into post-training quantization (PTQ) and quantization-aware training (QAT).

\noindent \textbf{Post-training quantization.}
PTQ applies quantization to pretrained models without further training. One line of work explores mixed-precision strategies~\cite{dettmers2022llmint88bitmatrixmultiplication, zhao2024atom}, which assign higher bit-widths to layers or tensors that are more sensitive to quantization error, particularly those affected by severe outliers. Another line of research focuses on scaling-based methods~\cite{xiao2024smoothquantaccurateefficientposttraining, shao2024omniquantomnidirectionallycalibratedquantization, wei2023outlier}, which balance the magnitudes of weights and activations via per-channel scaling and shifting to reduce quantization error. However, such scaling strategies often increase the complexity of weight quantization and tend to degrade significantly under extremely low-bit settings. Inspired by SmoothQuant~\cite{xiao2024smoothquantaccurateefficientposttraining}, AWQ~\cite{lin2024awq} proposes an activation-aware per-channel scaling method to reduce weight quantization error. More recently, rotation-based approaches\cite{ashkboos2024quarotoutlierfree4bitinference, liu2025spinquantllmquantizationlearned} have demonstrated strong effectiveness: by applying orthogonal or Hadamard transformations, activation outliers can be redistributed across channels through matrix multiplication, thereby alleviating extreme values and improving quantization robustness. Notably, QuaRot~\cite{ashkboos2024quarotoutlierfree4bitinference} is the first work to introduce Hadamard transforms for LLMs.

\noindent \textbf{Quantization-aware training.}
Quantization-Aware Training (QAT) methods~\cite{gupta2015deep, jacob2018quantization, esser2020learnedstepsizequantization, bhalgat2020lsqimprovinglowbitquantization, nagel2022overcoming} simulate quantization effects during the training phase, allowing
the model to find more optimal solutions compared to Post-Training Quantization (PTQ). However, traditional QAT approaches are severely bottlenecked by prolonged training times, massive memory footprints, and a strict reliance on labeled data and complex hyperparameter tuning. These formidable computational overheads render them largely impractical for modern Large Language Models (LLMs). To mitigate this, recent studies~\cite{liu2024llm, du2024bitdistiller, chen2025efficientqat, dettmers2023qlora, xu2023qa, bondarenko2024low} have explored efficient QAT variants tailored for LLMs.

\section{Comparison with Other Initialization Methods}
\label{Appendix: act init}
For low-bit coarse-grained activation quantization (e.g., 4-bit per-tensor), it is crucial to identify a stable method for initializing quantization parameters in order to provide a favorable starting point for optimization and to promote convergence. Under such low-bit and coarse-grained settings, quantization performance is particularly vulnerable to outliers. Accordingly, we evaluate initialization methods based on mean statistics as well as various common clipping strategies. The resulting model perplexity (PPL) on the Wikitext2 dataset is reported in \cref{tab:mean_vs_clip}.

We further analyze different clipping algorithms with varying thresholds and observe that fixed thresholds are not universally applicable to all activation distributions. They may lead to either overly aggressive or insufficient clipping, both of which harm model performance. In contrast, the Mean-based initialization method, which dynamically adapts the clipping range to the distribution, achieves the best initialization effect. This finding also provides empirical support for the motivation behind introducing learnable clipping thresholds in the prior work\cite{sun2025flatquantflatnessmattersllm}.

Moreover, we observe that under low-bit coarse-grained activation quantization, omitting rotation substantially increases the quantization difficulty. In this case, none of the initialization methods yield satisfactory performance. In contrast, when rotation is applied, the Mean-based initialization achieves the best results. Overall, the Mean-based initialization method is justified and demonstrates strong effectiveness.

\begin{table*}[t]
\centering
\small
\caption{Matrix multiplication speedup under different quantization configurations.}
\label{tab:mm_overhead}
\begin{tabular}{c c | c c | c}
\toprule
A. Gran. & A. Prec. & W. Gran. & W. Prec. & Speedup \\
\midrule
-- & FP & -- & FP & $1.0\times$ \\
\midrule
per-tensor & 16-bit & per-block   & 4-bit & $1.67\times$ \\
per-tensor & 8-bit & per-channel & 4-bit & $2.0\times$ \\
per-tensor & 8-bit & per-tensor  & 4-bit & $2.0\times$ \\
\bottomrule
\end{tabular}
\end{table*}

\begin{table*}[t]
\centering
\small
\setlength{\tabcolsep}{4pt}
\renewcommand{\arraystretch}{1.15}
\caption{PPL on Wikitext2 for SmolLM2-1.7B-Instruct under 4-bit per-tensor activation quantization, comparing Mean-based initialization with various common clipping thresholds, with and without rotation.}
\label{tab:mean_vs_clip}
\begin{adjustbox}{max width=0.97\linewidth}
\begin{tabular}{c | c c c c} 
\toprule
Activation-Config & Mean & Clip(99\%) & Clip(99.9\%) & Clip(99.99\%)\\
\midrule
4-bit + with rotation & \cellcolor{lightred}30.27 & 2838.81 & 145.63 & 183.15 \\
4-bit + without rotation & 6550995.0 & 18466806.0 & 502306.38 & 47484904.0 \\
\bottomrule
\end{tabular}
\end{adjustbox}
\end{table*}

\section{Quantization Performance under different configurations}
\label{Appendix: rule}
As shown in Table~\ref{tab:Activation-Config}, for Llama3.2-3B-Instruct, we restrict quantization to the input activations of linear layers. This setting facilitates a systematic analysis of different quantization configurations, such as the presence or absence of rotation and the activation bit-width allowing a clearer evaluation of how various initialization strategies affect quantized model performance. Model performance is measured by perplexity (PPL) on the C4 dataset.

\begin{table*}[t]
\centering
\small
\setlength{\tabcolsep}{4pt}
\renewcommand{\arraystretch}{1.15}
\caption{Perplexity (PPL) on C4 for activation quantization under different configurations and initialization methods.}
\label{tab:Activation-Config}
\begin{adjustbox}{max width=0.97\linewidth}
\begin{tabular}{c | c c} 
\toprule
Activation-Config & Max-Min & Mean \\
\midrule

8-bit + without rotation & \cellcolor{lightred}33.24 & 12651.84 \\
4-bit + without rotation & 276217.28 & 12833.83 \\
8-bit + with rotation & \cellcolor{lightred}18.55 & \cellcolor{lightred}25.37 \\
4-bit + with rotation & 548.27 & \cellcolor{lightred}43.68 \\
\bottomrule
\end{tabular}
\end{adjustbox}
\end{table*}

We further validate this conclusion on weight quantization by repeating the same experimental settings and evaluations. The results, also measured by PPL on C4, exhibit the same pattern, and are summarized in Table~\ref{tab:Weight-Config}.

\begin{table*}[t]
\centering
\small
\setlength{\tabcolsep}{4pt}
\renewcommand{\arraystretch}{1.15}
\caption{Perplexity (PPL) on C4 for weight quantization under different configurations and initialization methods.}
\label{tab:Weight-Config}
\begin{adjustbox}{max width=0.97\linewidth}
\begin{tabular}{c | c c} 
\toprule
Weight-Config & Max-Min & Mean \\
\midrule

8-bit + without rotation & \cellcolor{lightred}17.44 & 1646.30 \\
4-bit + without rotation & 363402.22 & 3791.84 \\
8-bit + with rotation & \cellcolor{lightred}16.79 & \cellcolor{lightred}17.67 \\
4-bit + with rotation & 5157.11 & \cellcolor{lightred}22.38 \\

\bottomrule
\end{tabular}
\end{adjustbox}
\end{table*}

\section{Theoretical Explanation of Mean-Based Initialization with Rotation}
\label{Appendix: prove}
In this section, we provide a theoretical explanation and analysis of the quantization parameters initialization methodology introduced in \cref{sec:init}. Some conclusions are intuitive and do not require extensive derivation. For instance, under 8-bit quantization, the quantizer possesses sufficient representational capacity. Therefore, whether or not a rotation matrix is introduced, the quantization parameters can be initialized using the Max–Min method. In contrast, under 4-bit quantization, the representational capacity is limited, and even with a rotation matrix, Max–Min initialization still leads to substantial quantization error.

More importantly, we observe a key phenomenon: under both 4-bit and 8-bit quantization, the Mean-based initialization method becomes stable and effective only after the introduction of rotation. In the following, we provide a theoretical explanation for this observation.

\subsection{Effect of Orthogonal Rotation on Activation Statistics}

\begin{assumption}
\label{ass:xfinite}
Let the unrotated activation be flattened into a vector
$x \in \mathbb{R}^N$ with zero mean, i.e.,
$\mathrm{mean}(x) = 0$.
\end{assumption}

\begin{definition}[Orthogonal Transformation of Activations]
\label{def:inj}
Let $x \in \mathbb{R}^N$ be a random vector satisfying
$\mathbb{E}[x_i] = 0$ for all $i$.
Let $R \in \mathbb{R}^{N \times N}$ be an orthogonal matrix such that
$R^\top R = I$, and define $y = R x$.
\end{definition}

\begin{proposition}
\label{prop:rotation_stats}
Under Assumption~\ref{ass:xfinite} and Definition~\ref{def:inj},
the quantities $\mathrm{mean}(y) \pm 3\,\mathrm{std}(y)$
are equal to $\mathrm{mean}(x) \pm 3\,\mathrm{std}(x)$
in expectation.
\end{proposition}

\begin{proof}
We first analyze the mean of $y$.
By definition,
\[
\mathrm{mean}(y) = \frac{1}{N} \mathbf{1}^\top y
= \frac{1}{N} \mathbf{1}^\top R x .
\]
Let $v = R^\top \mathbf{1}$, so that $\mathbf{1}^\top R = v^\top$.
Then
\[
\mathrm{mean}(y) = \frac{1}{N} v^\top x .
\]
Assuming independence between $v_i$ and $x_i$, together with
$\mathbb{E}[x_i] = 0$, we obtain
\[ 
\mathbb{E}[\mathrm{mean}(y)] 
= \frac{1}{N} \mathbb{E}[v^T x]
= \frac{1}{N} \sum_{i=1}^N \mathbb{E}[v_i x_i] 
= \frac{1}{N} \sum_{i=1}^N \mathbb{E}[v_i] \mathbb{E}[x_i] 
= 0 
\]

Next, we consider the standard deviation.
By orthogonality of $R$, we have
\[
\|y\|_2^2 = \|R x\|_2^2 = x^\top R^\top R x = \|x\|_2^2 .
\]
The standard deviation of $x$ is given by
\[
\mathrm{std}(x)
= \sqrt{\frac{1}{N} \|x\|_2^2 - \mathrm{mean}(x)^2}.
\]
Since $\mathbb{E}[\mathrm{mean}(x)] = 0$, this yields
\[
\mathbb{E}[\mathrm{std}(x)]
= \sqrt{\frac{1}{N} \|x\|_2^2}
= \mathbb{E}[\mathrm{std}(y)] .
\]

Under the commonly adopted $3\sigma$ approximation of the dynamic range,
\[
\mathrm{max}(z) \approx \mathrm{mean}(z) + 3\,\mathrm{std}(z), \quad
\mathrm{min}(z) \approx \mathrm{mean}(z) - 3\,\mathrm{std}(z),
\]
the equality of mean and standard deviation implies that the effective quantization range is preserved in expectation.
\end{proof}

\begin{figure*}[t]
  \centering
  % ---------- Row 1 ----------
  \begin{subfigure}[t]{0.24\linewidth}
    \centering
    \includegraphics[width=\linewidth]{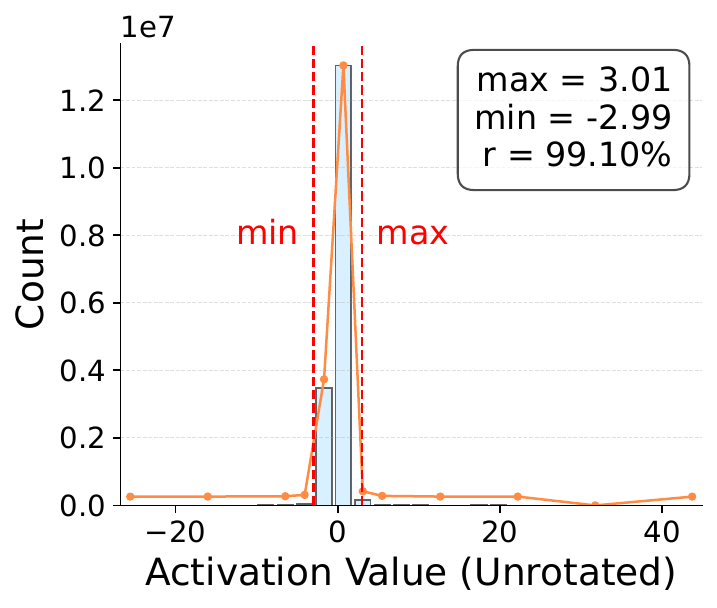}
    \caption{Unrotated activation distribution for $W_q$ of the $10^{th}$ transformer layer.}
    \label{fig:UR_1}
  \end{subfigure}
  \hfill
  \begin{subfigure}[t]{0.24\linewidth}
    \centering
    \includegraphics[width=\linewidth]{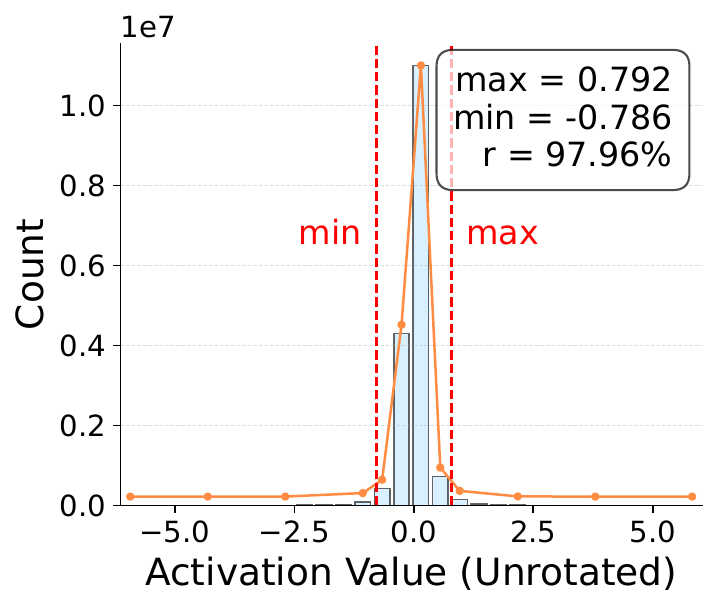}
    \caption{Unrotated activation distribution for $W_o$ of the $10^{th}$ transformer layer.}
    \label{fig:UR_2}
  \end{subfigure}
  \hfill
  \begin{subfigure}[t]{0.24\linewidth}
    \centering
    \includegraphics[width=\linewidth]{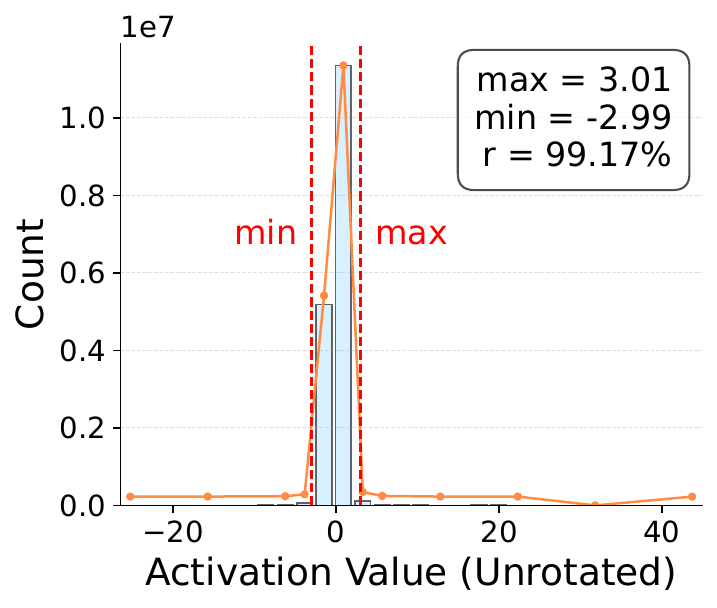}
    \caption{Unrotated activation distribution for $W_{up}$ of the $10^{th}$ transformer layer.}
    \label{fig:UR_3}
  \end{subfigure}
  \hfill
  \begin{subfigure}[t]{0.24\linewidth}
    \centering
    \includegraphics[width=\linewidth]{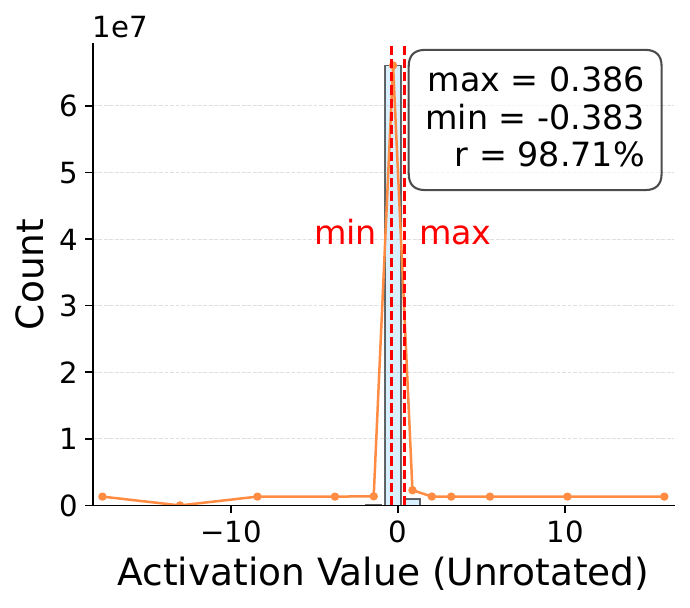}
    \caption{Unrotated activation distribution for $W_{down}$ of the $10^{th}$ transformer layer.}
    \label{fig:UR_4}
  \end{subfigure}

  \vspace{2mm} % 行间距，可调

  % ---------- Row 2 ----------
  \begin{subfigure}[t]{0.24\linewidth}
    \centering
    \includegraphics[width=\linewidth]{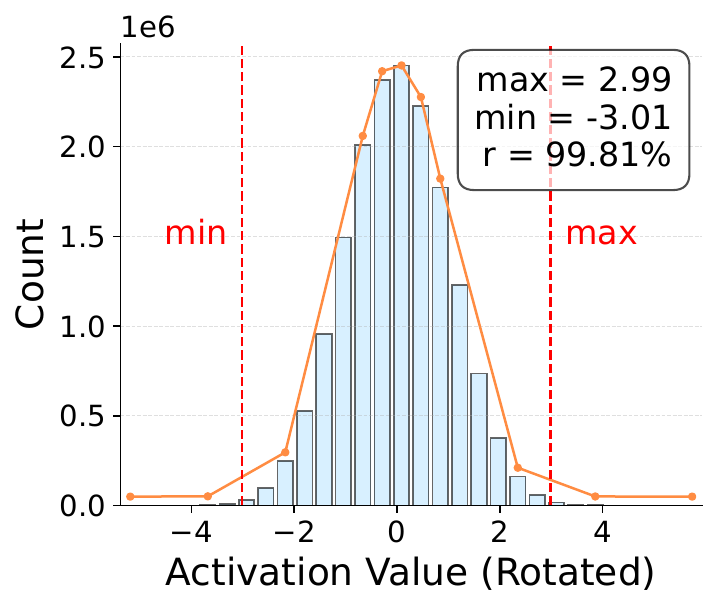}
    \caption{Rotated activation distribution for $W_q$ of the $10^{th}$ transformer layer.}
    \label{fig:R_1}
  \end{subfigure}
  \hfill
  \begin{subfigure}[t]{0.24\linewidth}
    \centering
    \includegraphics[width=\linewidth]{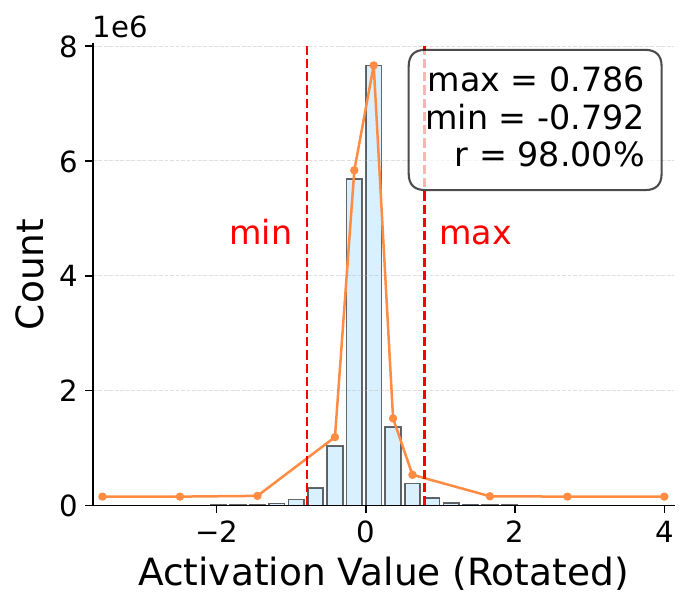}
    \caption{Rotated activation distribution for $W_o$ of the $10^{th}$ transformer layer.}
    \label{fig:R_2}
  \end{subfigure}
  \hfill
  \begin{subfigure}[t]{0.24\linewidth}
    \centering
    \includegraphics[width=\linewidth]{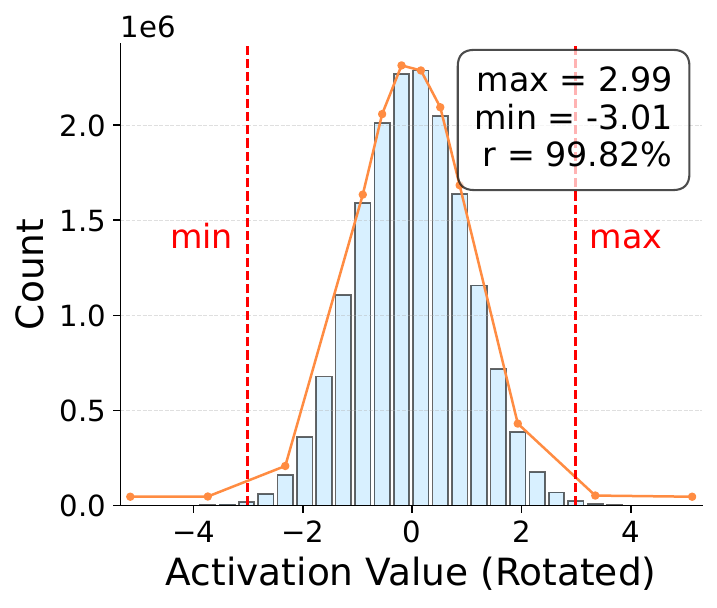}
    \caption{Rotated activation distribution for $W_{up}$ of the $10^{th}$ transformer layer.}
    \label{fig:R_3}
  \end{subfigure}
  \hfill
  \begin{subfigure}[t]{0.24\linewidth}
    \centering
    \includegraphics[width=\linewidth]{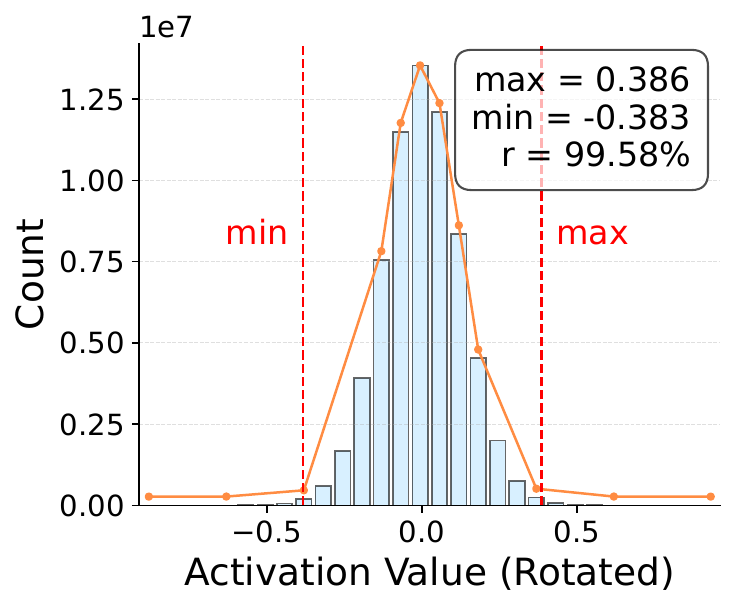}
    \caption{Rotated activation distribution for $W_{down}$ of the $10^{th}$ transformer layer.}
    \label{fig:R_4}
  \end{subfigure}

  \caption{
  Activation distributions of $W_q$, $W_o$, $W_{up}$, and $W_{down}$ in the $10^{th}$ transformer layer of SmolLM2-1.7B-Instruct. ~\cref{fig:UR_1}-~\cref{fig:UR_4} show the unrotated distributions, while ~\cref{fig:R_1}-~\cref{fig:R_4} show the rotated distributions. Blue bars represent the activation distributions. The labels \textit{min} and \textit{max} indicate the boundaries of the quantization range (calculated based on the Mean-based initialization method), and the box in the top-right corner displays the \textit{min/max} values along with the percentage of activations that fall within this range.
  }
  \label{fig:R_UR_8}
\end{figure*}

\subsection{Advantages of Mean-Based Initialization with Rotation}
~\cref{ass:xfinite} requires the activation distribution to have a mean close to zero. However, by visualizing the pre-rotation activations of LLaMA-3.2-3B-Instruct, Qwen2.5-3B-Instruct, and SmolLM2-1.7B-Instruct, we observe that the output activations of the gate\_proj have a mean that significantly deviates from zero, especially in Qwen models. Aside from this exception, all other activations satisfy the assumptions in ~\cref{ass:xfinite}. Therefore, in our subsequent analysis, we exclude gate\_proj activations from low-bit quantization and retain them in BF16 precision. 

To provide a more intuitive illustration of ~\cref{prop:rotation_stats}, we visualize the activation distributions of $W_q$, $W_o$, $W_{up}$, and $W_{down}$ in the $10^{th}$ transformer layer of SmolLM2-1.7B-Instruct, as shown in ~\cref{fig:R_UR_8}. From the visualization, we summarize the advantages of combining Mean-based initialization with rotation as follows.

For activations within the quantization coverage range, the numerical values are relatively small, allowing the dequantized values to closely approximate the original distribution regardless of whether a rotation is applied. Consequently, the quantization error for this subset is typically concentrated near zero, indicating high quantization accuracy. In contrast, activations that fall outside the coverage range (i.e., outliers) incur substantial quantization errors. Without a rotation matrix, these outliers have large magnitudes, and the information they carry is severely truncated during quantization, leading to noticeable performance degradation. Introducing a rotation matrix smooths the distribution, mapping more activations into the coverage range and reducing the magnitude of outliers, thereby mitigating their impact. Hence, Mean-based initialization becomes effective only after rotation; applying it directly to the unrotated activations would result in large quantization errors.

\begin{figure*}[t]
  \centering
  \begin{subfigure}[t]{0.24\linewidth}
    \centering
    \includegraphics[width=\linewidth]{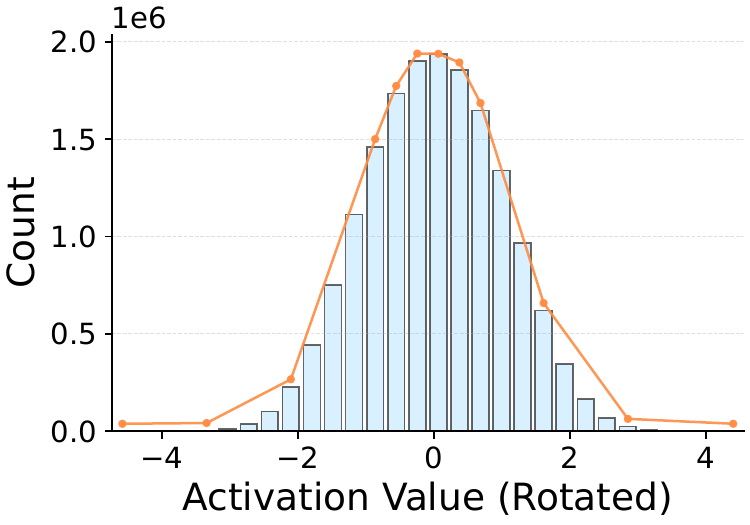}
    \caption{Rotated activation distribution for $W_q$ of the $19^{th}$ transformer layer.}
    \label{fig:q_19}
  \end{subfigure}
  \hfill
  \begin{subfigure}[t]{0.24\linewidth}
    \centering
    \includegraphics[width=\linewidth]{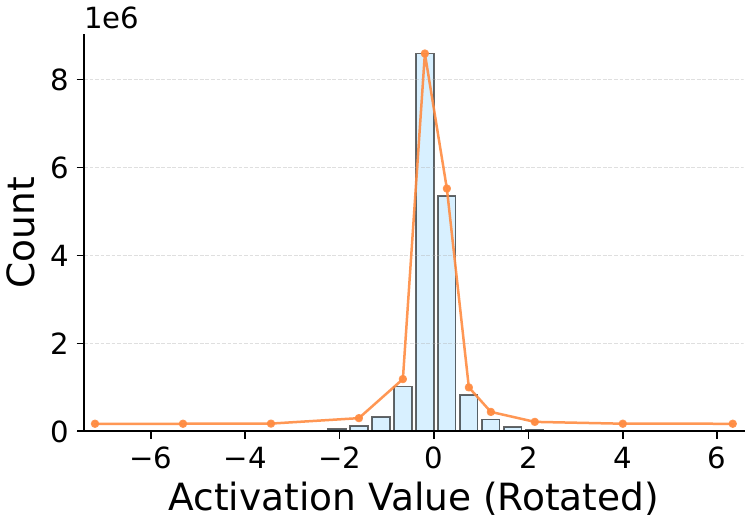}
    \caption{Rotated activation distribution for $W_o$ of the $19^{th}$ transformer layer.}
    \label{fig:o_19}
  \end{subfigure}
  \hfill
  \begin{subfigure}[t]{0.24\linewidth}
    \centering
    \includegraphics[width=\linewidth]{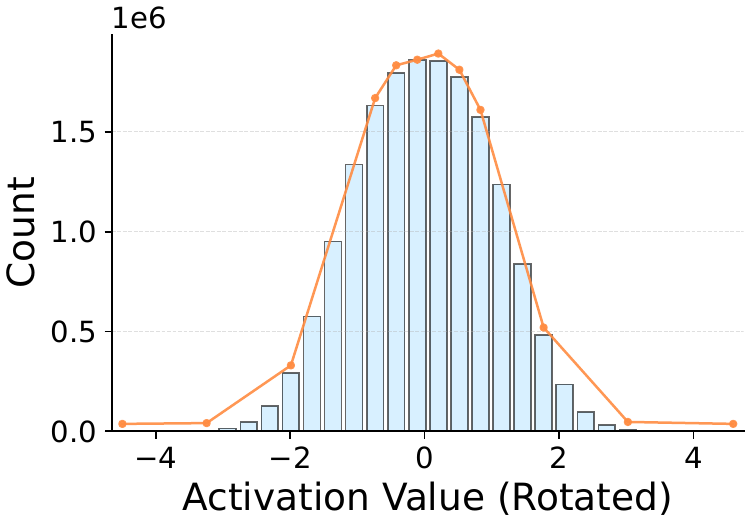}
    \caption{Rotated activation distribution for $W_{up}$ of the $19^{th}$ transformer layer.}
    \label{fig:up_19}
  \end{subfigure}
   \hfill
  \begin{subfigure}[t]{0.24\linewidth}
    \centering
    \includegraphics[width=\linewidth]{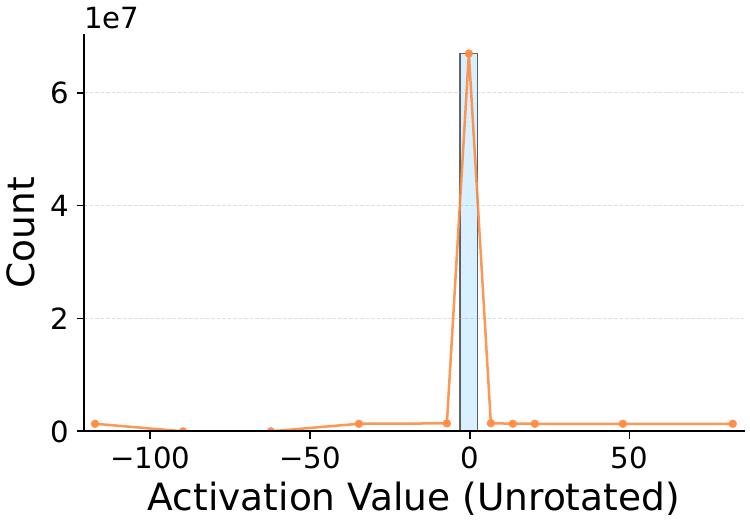}
    \caption{Unrotated activation distribution for $W_{down}$ of the $19^{th}$ transformer layer.}
    \label{fig:down_19}
  \end{subfigure}
  
  \caption{
  Activation distributions of $W_q$, $W_o$, $W_{up}$, and $W_{down}$ in the $19^{th}$ transformer layer of SmolLM2-1.7B-Instruct. ~\cref{fig:q_19}-~\cref{fig:up_19} show the rotated distributions of $W_q$, $W_o$, and $W_{up}$, while ~\cref{fig:down_19} shows the unrotated distributions of $W_{down}$. 
  }
  \label{fig:down_outlier}
\end{figure*}

\section{Theoretical Analysis on the Failure of Gradient-Based Optimization for Unrotated Tensors}
\label{Appendix: Gradient}

In ~\cref{sec:train}, we point out that directly applying gradient-based optimization to the quantization parameters of unrotated tensors introduces significant instability, often resulting in worse model accuracy than using unoptimized parameters. In this section, we further analyze the fundamental nature of quantization error optimization and provide a theoretical analysis for the optimization failure caused by the conflict between clipping error and rounding error under heavy-tailed distributions.

The optimization of quantization parameters essentially aims to find a theoretical equilibrium point that minimizes the total quantization error, denoted as
\begin{align} 
E_{\text{total}}(s)=E_{\text{rounding}}(s)+E_{\text{clipping}}(s).
\end{align} 
For a random input tensor, the total quantization error can be rigorously decomposed into two components: the rounding error, produced when values are mapped to the nearest quantization grid, and the clipping error, caused by truncating values that exceed the representable integer range. The quantization scale \(s\) directly governs the trade-off between these two types of errors.

For rotated tensors that approximately follow a Gaussian distribution, most values are concentrated around the mean. Under such distributions, the trade-off curve between \(E_{\text{rounding}}\) and \(E_{\text{clipping}}\) varies smoothly as \(s\) changes. Consequently, the optimizer can relatively easily locate a global optimum. As a result, the quantization parameters can be optimized in a stable manner in this setting.

In contrast, unrotated tensors typically exhibit heavy-tailed distributions, where the majority of values are densely concentrated near zero while a small number of outliers are distributed far from the center. During optimization, reducing \(E_{\text{clipping}}\) for these extreme outliers forces the optimizer to enlarge the quantization scale \(s\). However, increasing \(s\) compresses the densely populated central region into only a few quantization bins near zero, leading to a dramatic increase in \(E_{\text{rounding}}\). Conversely, if the optimizer reduces \(s\) to preserve the quantization precision of the dominant central values and minimize \(E_{\text{rounding}}\), the large-magnitude outliers become severely clipped, significantly increasing \(E_{\text{clipping}}\).

Therefore, under the extremely limited representation capacity of low-bit quantization, the severe polarization inherent in heavy-tailed distributions causes the loss landscape of \(E_{\text{total}}(s)\) to contain numerous local minima. Consequently, gradient-based optimization through backpropagation becomes highly unstable. This observation further motivates the necessity of adopting a distribution-aware selective optimization strategy for tensors with different statistical characteristics, thereby mitigating the adverse effects introduced by heavy-tailed distributions during joint optimization.

\section{Down\_proj: Outlier-Dominated Activation Distribution}
\label{Appendix: down_proj outlier}
For the SmolLM2-1.7B-Instruct model, we visualized the distributions of input activations for q\_proj (k\_proj and v\_proj), o\_proj, up\_proj (gate\_proj), and down\_proj in \sys. To avoid the inference latency issue discussed in \cref{sec:bkg:r-qat}, we disabled all online rotations (R4) in our method. As a result, the input activations of down\_proj remain in their original, unrotated form, whereas the input activations of the remaining linear layers are smoothed via rotation.

As shown in ~\cref{fig:down_outlier}, and consistent with prior work \cite{xu2024fastondevicellminference}, the input activations of down\_proj layers exhibit particularly severe outliers: the numerical range is extensive, and a large portion of values concentrates near zero. To preserve quantization accuracy for these near-zero values, important information carried by the outliers may be lost. This ultimately leads to a substantial drop in model accuracy.

\section{Additional Experimental Results}
\label{Appendix: Additional Experimental}

\begin{table*}[!t]
\centering
\small
\setlength{\tabcolsep}{4pt}
\renewcommand{\arraystretch}{1.15}
\caption{\textbf{Evaluation on additional widely used models (Qwen2.5-3B-Instruct and SmolLM2-1.7B-Instruct)}: Perplexity(PPL) and Zero-shot QA task accuracy results of 4-bit weight or 8-bit weight quantized models.}
\label{tab:zero_shot_qa_qwen_smollm}
\begin{adjustbox}{max width=0.97\linewidth}
\begin{tabular}{c | c c c | c | c c c c c c c} 
\toprule
Model & Method & W-A-KV & W. Gran. & PPL(C4) $\downarrow$ & PIQA & Winogrande & HellaSwag & ARC-E & ARC-C & LAMBADA & Avg $\uparrow$ \\
\midrule

% ===================== Qwen2.5-3B-Instruct =====================
\multirow{14}{*}{\makecell{Qwen2.5\\-3B-it}}
& FP32       & --    & -- & 13.84 & 0.7807 & 0.6914 & 0.7330 & 0.5585 & 0.4249 & 0.5808 & 0.6282 \\
& executorch & 4-16-8 & per-blcok & 15.48 & 0.7606 & 0.6788 & 0.7172 & 0.5522 & 0.4317 & 0.5579 & 0.6164 \\
& executorch  & 8-8-8 & per-channel & 14.68 & 0.7497 & 0.6725 & 0.7203 & \textbf{0.5619} & 0.4181 & 0.5707 & 0.6155 \\
& QuaRot  & 8-8-8 & per-channel & 14.08 & \textbf{0.7824} & \textbf{0.6898} & \textbf{0.7320} & 0.5577 & 0.4164 & 0.5667 & 0.6242 \\
& SpinQuant  & 8-8-8 & per-channel & 14.07 & 0.7737 & 0.6859 & 0.7294 & 0.5421 & 0.4172 & \textbf{0.5824} & 0.6218 \\
& \cellcolor{lightred}Quant.npu & \cellcolor{lightred}8-8-8 & \cellcolor{lightred}per-channel & \cellcolor{lightred}\textbf{13.82} & \cellcolor{lightred}0.7764 & \cellcolor{lightred}0.6867 & \cellcolor{lightred}0.7278 & \cellcolor{lightred}0.5501 & \cellcolor{lightred}\textbf{0.4309} & \cellcolor{lightred}0.5750 & \cellcolor{lightred}\textbf{0.6245} \\
& executorch  & 4-8-8 & per-channel & 28901.54 & 0.5033 & 0.4901 & 0.2518 & 0.2483 & 0.2253 & 0.0 & 0.2865 \\
& QuaRot  & 4-8-8 & per-channel & 19.96 & 0.7220 & 0.6046 & 0.6450 & 0.4815 & 0.3626 & 0.3827 & 0.5331 \\
& SpinQuant  & 4-8-8 & per-channel & 20.61 & 0.7334 & 0.6109 & 0.6414 & \textbf{0.5808} & \textbf{0.4078} & 0.4017 & 0.5627 \\
& \cellcolor{lightred}Quant.npu & \cellcolor{lightred}4-8-8 & \cellcolor{lightred}per-channel & \cellcolor{lightred}\textbf{16.02} & \cellcolor{lightred}\textbf{0.7633} & \cellcolor{lightred}\textbf{0.6685} & \cellcolor{lightred}\textbf{0.6920} & \cellcolor{lightred}0.5101 & \cellcolor{lightred}0.4061 & \cellcolor{lightred}\textbf{0.5053} & \cellcolor{lightred}\textbf{0.5906} \\
& executorch  & 4-4-8 & per-channel & 415330.38 & 0.5000 & 0.5114 & 0.2537 & 0.2424 & 0.2346 & 0.0 & 0.2904 \\
& QuaRot  & 4-4-8 & per-channel & 97.01 & 0.5919 & 0.5043 & 0.3827 & 0.3565 & 0.2517 & 0.0507 & 0.3563 \\
& SpinQuant  & 4-4-8 & per-channel & 101.85 & 0.5860 & 0.5020 & 0.3683 & 0.3418 & 0.2628 & 0.0753 & 0.3560 \\
& \cellcolor{lightred}Quant.npu & \cellcolor{lightred}4-4-8 & \cellcolor{lightred}per-channel & \cellcolor{lightred}\textbf{17.96} & \cellcolor{lightred}\textbf{0.7437} & \cellcolor{lightred}\textbf{0.6227} & \cellcolor{lightred}\textbf{0.6619} & \cellcolor{lightred}\textbf{0.5105} & \cellcolor{lightred}\textbf{0.3874} & \cellcolor{lightred}\textbf{0.4440} &
\cellcolor{lightred}\textbf{0.5617} \\
\midrule    

% ===================== SmolLM2-1.7B-Instruct =====================
\multirow{14}{*}{\makecell{SmolLM2\\-1.7B-it}}
& FP32       & --    & -- & 13.55 & 0.7622 & 0.6835 & 0.7127 & 0.5417 & 0.4181 & 0.5703 & 0.6148 \\
& executorch & 4-16-8 & per-block & 15.35 & 0.7595 & 0.6488 & 0.6933 & 0.5362 & 0.4044 & 0.5178 & 0.5933 \\
& executorch  & 8-8-8 & per-channel & 25.75 & 0.7062 & 0.6188 & 0.6256 & 0.5194 & 0.3933 & 0.3889 & 0.5420 \\
& QuaRot  & 8-8-8 & per-channel & \textbf{13.93} & \textbf{0.7633} & 0.6646 & \textbf{0.7047} & 0.5471 & 0.4036 & 0.5525 & 0.6060 \\
& SpinQuant  & 8-8-8 & per-channel & 13.96 & 0.7628 & \textbf{0.6677} & 0.7020 & \textbf{0.5551} & \textbf{0.4172} & \textbf{0.5616} & \textbf{0.6111} \\
& \cellcolor{lightred}Quant.npu & \cellcolor{lightred}8-8-8 & \cellcolor{lightred}per-channel & \cellcolor{lightred}13.94 & \cellcolor{lightred}0.7601 & \cellcolor{lightred}0.6551 & \cellcolor{lightred}0.6859 & \cellcolor{lightred}0.5396 & \cellcolor{lightred}0.3942 & \cellcolor{lightred}0.5484 & \cellcolor{lightred}0.5972 \\
& executorch  & 4-8-8 & per-channel & 13121503.0 & 0.6202 & 0.5059 & 0.4470 & 0.4171 & 0.3012 & 0.0444 & 0.3893 \\
& QuaRot  & 4-8-8 & per-channel & 120.06 & 0.5794 & 0.5004 & 0.3465 & 0.3283 & 0.2474 & 0.1116 & 0.3523 \\
& SpinQuant  & 4-8-8 & per-channel& 67.05 & 0.6317 & 0.5257 & 0.4105 & 0.3662 & 0.2935 & 0.0974 & 0.3875 \\
& \cellcolor{lightred}Quant.npu & \cellcolor{lightred}4-8-8 & \cellcolor{lightred}per-channel & \cellcolor{lightred}\textbf{18.61} & \cellcolor{lightred}\textbf{0.7312} & \cellcolor{lightred}\textbf{0.6156} & \cellcolor{lightred}\textbf{0.6096} & \cellcolor{lightred}\textbf{0.4865} & \cellcolor{lightred}\textbf{0.3584} & \cellcolor{lightred}\textbf{0.4562} & \cellcolor{lightred}\textbf{0.5436} \\
& executorch  & 4-4-8 & per-channel & 53274784.0 & 0.5016 & 0.5004 & 0.2521 & 0.2500 & 0.2568 & 0.0 & 0.2935 \\
& QuaRot  & 4-4-8 & per-channel & 8553.56 & 0.5163 & 0.5020 & 0.2568 & 0.2563 & 0.2594 & 0.0006 & 0.2986 \\
& SpinQuant  & 4-4-8 & per-channel & 4283.87 & 0.4989 & 0.4893 & 0.2499 & 0.2529 & 0.2261 & 0.0004 & 0.2862 \\
& \cellcolor{lightred}Quant.npu & \cellcolor{lightred}4-4-8 & \cellcolor{lightred}per-channel & \cellcolor{lightred}\textbf{23.56} & \cellcolor{lightred}\textbf{0.6948} & \cellcolor{lightred}\textbf{0.5770} & \cellcolor{lightred}\textbf{0.5342} & \cellcolor{lightred}\textbf{0.4541} & \cellcolor{lightred}\textbf{0.3251} & \cellcolor{lightred}\textbf{0.3497} &
\cellcolor{lightred}\textbf{0.4898} \\
\bottomrule
\end{tabular}
\end{adjustbox}
\end{table*}

\subsection{Results on Additional Widely Used Models}

To further demonstrate the generalizability of \sys, we extend our analysis to other widely deployed models, including Qwen2.5-3B-Instruct~\cite{qwen2025qwen25technicalreport} and SmolLM2-1.7B-Instruct~\cite{allal2025smollm2smolgoesbig}. \cref{tab:zero_shot_qa_qwen_smollm} reports the perplexity (PPL) on C4 and the accuracy on six zero-shot commonsense reasoning benchmarks under W8A8, W4A8, and W4A4 settings. Consistent with the results in ~\cref{Main Results}, \sys achieves zero-shot accuracy comparable to the FP32 baseline under W8A8 quantization on Qwen2.5-3B-Instruct. The performance gains are particularly significant under the W4A8 setting, where \sys achieves a maximum accuracy improvement of up to 30.41\% compared to baseline methods. Moreover, \sys-W4A4 maintains a relatively high accuracy, even in scenarios where the baseline almost entirely fails. For Qwen2.5-3B-Instruct, \sys exhibits only about a 6\% drop in accuracy compared to the FP32 baseline. Furthermore, \sys demonstrates broad applicability by yielding comparable performance gains on SmolLM2-1.7B-Instruct.

\begin{table*}[t]
\centering
\small
\setlength{\tabcolsep}{4pt}
\renewcommand{\arraystretch}{1.15}
\caption{\textbf{Evaluation on recent architectures (Qwen3-1.7B)}: Perplexity(PPL) and Zero-shot QA task accuracy results of 4-bit weight or 8-bit weight quantized models.}
\label{tab:qwen3_results}
\begin{adjustbox}{max width=0.97\textwidth}
\begin{tabular}{c | c c c | c | c c c c c c c} 
\toprule
Model & Method & W-A-KV & W. Gran. & PPL(C4) $\downarrow$ & PIQA & Winogrande & HellaSwag & ARC-E & ARC-C & LAMBADA & Avg $\uparrow$ \\
\midrule

% ===================== Qwen3-1.7B =====================
\multirow{14}{*}{\makecell{Qwen3\\-1.7B}}
& FP32       & --    & -- & 22.64 & 0.7160 & 0.6093 & 0.5932 & 0.5438 & 0.3993 & 0.4731 & 0.5558 \\
& executorch & 4-16-8 & per-block & 31.52 & 0.6654 & 0.5596 & 0.5384 & 0.4844 & 0.3473 & 0.3196 & 0.4858 \\
& executorch  & 8-8-8 & per-channel & 29.99 & 0.6589 & 0.5588 & 0.5294 & 0.4907 & 0.3447 & 0.3468 & 0.4882 \\
& QuaRot  & 8-8-8 & per-channel & 28.91 & 0.6659 & 0.5706 & 0.5456 & 0.5316 & 0.3703 & 0.3833 & 0.5112 \\
& SpinQuant  & 8-8-8 & per-channel & 28.64 & 0.6659 & 0.5880 & 0.5468 & 0.5219 & 0.3601 & 0.3804 & 0.5105 \\
& \cellcolor{lightred}Quant.npu & \cellcolor{lightred}8-8-8 & \cellcolor{lightred}per-channel & \cellcolor{lightred}\textbf{22.74} & \cellcolor{lightred}\textbf{0.6872} & \cellcolor{lightred}\textbf{0.6014} & \cellcolor{lightred}\textbf{0.5710} & \cellcolor{lightred}\textbf{0.5354} & \cellcolor{lightred}\textbf{0.3959} & \cellcolor{lightred}\textbf{0.4430} & \cellcolor{lightred}\textbf{0.5390} \\
& executorch  & 4-8-8 & per-channel & 72.61 & 0.6094 & 0.5059 & 0.4179 & 0.4061 & 0.2705 & 0.1230 & 0.3888 \\
& QuaRot  & 4-8-8 & per-channel & 429.09 & 0.5783 & 0.5059 & 0.3546 & 0.3409 & 0.2602 & 0.1250 & 0.3608 \\
& SpinQuant  & 4-8-8 & per-channel & 548.10 & 0.5637 & 0.5091 & 0.3364 & 0.3173 & 0.2457 & 0.1209 & 0.3489 \\
& \cellcolor{lightred}Quant.npu & \cellcolor{lightred}4-8-8 & \cellcolor{lightred}per-channel & \cellcolor{lightred}\textbf{31.68} & \cellcolor{lightred}\textbf{0.6779} & \cellcolor{lightred}\textbf{0.5604} & \cellcolor{lightred}\textbf{0.5054} & \cellcolor{lightred}\textbf{0.4651} & \cellcolor{lightred}\textbf{0.3311} & \cellcolor{lightred}\textbf{0.2729} & \cellcolor{lightred}\textbf{0.4688} \\
& executorch  & 4-4-8 & per-channel & 491615.22 & 0.4924 & 0.5091 & 0.2506 & 0.2551 & 0.2440 & 0.0 & 0.2919 \\
& QuaRot  & 4-4-8 & per-channel & 6898.98 & 0.4913 & 0.4925 & 0.2717 & 0.2681 & 0.2534 & 0.0099 & 0.2978 \\
& SpinQuant  & 4-4-8 & per-channel & 5018.61 & 0.5169 & \textbf{0.5328} & 0.2760 & 0.2727 & 0.2534 & 0.0151 & 0.3112 \\
& \cellcolor{lightred}Quant.npu & \cellcolor{lightred}4-4-8 & \cellcolor{lightred}per-channel & \cellcolor{lightred}\textbf{36.25} & \cellcolor{lightred}\textbf{0.6279} & \cellcolor{lightred}0.5241 & \cellcolor{lightred}\textbf{0.4677} & \cellcolor{lightred}\textbf{0.4693} & \cellcolor{lightred}\textbf{0.3268} & \cellcolor{lightred}\textbf{0.2199} &
\cellcolor{lightred}\textbf{0.4393} \\
\bottomrule
\end{tabular}
\end{adjustbox}
\end{table*}

\subsection{Results on More Recent Architectures}

To further evaluate the generality of \sys on recent model architectures, we additionally conduct optimization and evaluation on Qwen3-1.7B~\cite{yang2025qwen3} and compare our method against relevant baselines. The detailed results are provided in ~\cref{tab:qwen3_results}.

As shown in ~\cref{tab:qwen3_results}, \sys consistently achieves superior accuracy preservation across different quantization settings. Under the highly constrained W4A4KV8 per-channel setting, \sys achieves an average accuracy of 0.4393 with a PPL(C4) of 36.25, significantly outperforming the baseline, SpinQuant, which suffers severe degradation (Avg: 0.3112, PPL: 5018.61). Even under the W8A8KV8 setting, \sys achieves an average accuracy of 0.5390 and a PPL(C4) of 22.74, closely matching the FP32 model (Avg: 0.5558, PPL: 22.64). These results demonstrate that \sys can be effectively generalized to newer architectures without requiring architecture-specific modifications.

In addition, \sys is practical and easy to deploy. For Qwen3-1.7B, the complete quantization pipeline requires only approximately 2.5 hours on NVIDIA A40 GPUs under the experimental setup described in ~\cref{Experiments}. Adapting \sys to a new model requires no complex model-specific engineering, demonstrating the practicality and portability of our method.

\begin{table*}[t]
\centering
\small
\setlength{\tabcolsep}{4pt}
\renewcommand{\arraystretch}{1.15}
\caption{\textbf{Evaluation on larger LLMs (LLaMA3-8B)}: Perplexity(PPL) and Winogrande accuracy results of 4-bit weight  quantized models.}
\label{tab:llama8b_results}
\begin{adjustbox}{max width=0.97\textwidth}
\begin{tabular}{c | c c c | c c} 
\toprule
Model & Method & W-A-KV & W. Gran. & PPL(C4) $\downarrow$ & Winogrande $\uparrow$ \\
\midrule

% ===================== Qwen3-1.7B =====================
\multirow{6}{*}{\makecell{LLaMA3\\-8B}}
& FP32       & --    & -- & 9.07 & 0.7285 \\
& executorch & 4-16-8 & per-block & 10.35 & 0.7182 \\
& executorch  & 4-8-8 & per-channel & 17.93 & 0.6756 \\
& QuaRot  & 4-8-8 & per-channel & 19.15 & 0.6819 \\
& SpinQuant  & 4-8-8 & per-channel & 22.37 & 0.6448 \\
& \cellcolor{lightred}Quant.npu & \cellcolor{lightred}4-8-8 & \cellcolor{lightred}per-channel & \cellcolor{lightred}11.88 & \cellcolor{lightred}0.7190 \\
\bottomrule
\end{tabular}
\end{adjustbox}
\end{table*}

\subsection{Results on Larger Models}

Our primary evaluation focuses on 1B--3B models, which align with the target deployment scenario of mobile devices with limited memory budgets (typically around 4GB available memory). Accordingly, we evaluates \sys on several representative mobile-scale models, including Llama-3.2-3B-Instruct, Qwen2.5-3B-Instruct, and SmolLM2-1.7B-Instruct.

To further validate the scalability of \sys, we additionally conduct experiments on the substantially larger Llama3-8B model~\cite{grattafiori2024llama3herdmodels}. The detailed results are summarized in ~\cref{tab:llama8b_results}.

As shown in ~\cref{tab:llama8b_results}, \sys consistently maintains strong quantization performance at larger scales. Under the challenging W4A8KV8 per-channel configuration, baseline methods suffer significant degradation, with ExecuTorch, QuaRot, and SpinQuant achieving PPL values of 17.93, 19.15, and 22.37, respectively. In contrast, \sys substantially reduces the PPL to 11.88 while achieving an accuracy of 0.7190 on Winogrande. Notably, this performance not only significantly surpasses all other W4A8 per-channel baselines, but also approaches the higher-precision ExecuTorch-W4A16 per-block setting (PPL: 10.35) and the FP32 model (PPL: 9.07).

These results demonstrate that \sys remains highly robust across substantially different model scales. More importantly, as model capacity increases, \sys exhibits the consistent capability in preserving model accuracy under aggressive NPU-friendly quantization constraints.

\begin{table*}[t]
\centering
\small
\setlength{\tabcolsep}{4pt}
\renewcommand{\arraystretch}{1.15}
\caption{\textbf{Comparison with MobileQuant}: Perplexity (PPL) and HellaSwag accuracy results of 4-bit weight quantization (w4a8, with per-channel weight quantization) on the TinyLlama-1.1B-Chat-v1.0 model.}
\label{tab:mobilequant_results}
\begin{adjustbox}{max width=0.97\textwidth}
\begin{tabular}{c | c c} 
\toprule
Method & PPL(C4) $\downarrow$ & HellaSwag $\uparrow$ \\
\midrule
FP32 & 9.75 & 0.5087 \\
MoblieQuant & 11.95 & 0.4787 \\
\cellcolor{lightred}Quant.npu & \cellcolor{lightred}11.43 & \cellcolor{lightred}0.4893 \\
\bottomrule
\end{tabular}
\end{adjustbox}
\end{table*}

\subsection{Additional Comparison with MobileQuant}

Our original evaluation primarily focused on rotation-based quantization methods, including QuaRot and SpinQuant, since these approaches consistently demonstrate stronger accuracy preservation compared with earlier methods such as SmoothQuant and OmniQuant, upon which MobileQuant~\cite{tan2024mobilequant} is built. To provide a more comprehensive evaluation, we additionally conduct a direct comparison between \sys and MobileQuant under the exact same static NPU-friendly quantization constraints. We conduct this evaluation on the TinyLlama-1.1B-Chat-v1.0 model. The detailed performance metrics are summarized in Table~\ref{tab:mobilequant_results}.

As shown in ~\cref{tab:mobilequant_results}, \sys consistently outperforms MobileQuant across both perplexity and downstream task accuracy. Specifically, MobileQuant achieves a PPL of 11.95 and an accuracy of 0.4787 on HellaSwag, whereas \sys improves the PPL to 11.43 and achieves a higher accuracy of 0.4893 on HellaSwag, substantially narrowing the gap toward the FP32 baseline (PPL: 9.75, HellaSwag: 0.5087). This further highlights the performance advantage of \sys under fully static NPU-friendly quantization constraints.

\begin{table*}[t]
\centering
\small
\setlength{\tabcolsep}{4pt}
\renewcommand{\arraystretch}{1.15}
\caption{Instruction-following evaluation on Llama-3.2-3B-Instruct using AlpacaEval 2.0. Win rates are computed against the FP16 model to isolate the impact of quantization.}
\label{tab:AlpacaEval}
\begin{adjustbox}{max width=0.97\textwidth}
\begin{tabular}{c | c c c} 
\toprule
Model & length\_controlled\_winrate $\uparrow$ & win\_rate $\uparrow$ & avg\_length \\
\midrule

Executorch-w4a8 & 9.09 & 9.59 & 2301 \\
Quant.npu-w4a8 & 24.76 & 22.43 & 2167 \\
\cellcolor{lightred}Quant.npu-w8a8 & \cellcolor{lightred}48.82 & \cellcolor{lightred}53.59 & \cellcolor{lightred}2406 \\
\bottomrule
\end{tabular}
\end{adjustbox}
\end{table*}

\subsection{Instruction-Following Capability Evaluation}

To further evaluate whether \sys preserves instruction-following capability after quantization, we conduct additional experiments on Llama-3.2-3B-Instruct using AlpacaEval 2.0.

Since small models typically exhibit relatively low win rates against the default GPT-4-Turbo reference model in AlpacaEval 2.0, direct comparison against GPT-4-Turbo can obscure the effect of quantization. For a clearer assessment, we instead use the FP16 Llama-3.2-3B-Instruct model as the reference model and compute the win rates of quantized variants against it.

As shown in Table~\ref{tab:AlpacaEval}, \sys-W8A8 achieves a length-controlled win rate of 48.82\%, indicating nearly lossless generation quality compared with the FP16 model. In contrast, the more aggressive W4A8 setting still achieves a 24.76\% length-controlled win rate, demonstrating that \sys maintains reasonable conversational quality even under low-bit constraints.

Notably, \sys significantly outperforms the ExecuTorch-W4A8 baseline, which achieves only a 9.09\% length-controlled win rate. This substantial gap highlights the effectiveness of our method in preserving instruction-following capability under aggressive quantization.

\begin{table*}[t]
\centering
\small
\setlength{\tabcolsep}{4pt}
\renewcommand{\arraystretch}{1.15}
\caption{\textbf{Detailed performance comparison between Quant.npu and ExecuTorch}: Prefill and decode speed, memory footprint, and energy consumption of different LLMs across various datasets under different quantization settings.}
\label{tab:llm_speed_energy}

% ================= (a) =================
\begin{subtable}[t]{\textwidth}
\centering
\caption{Executorch: 4-bit per-block weight quantization}
\begin{adjustbox}{max width=0.97\textwidth}
\begin{tabular}{c c | c c c c | c c c c} 
\toprule
Model & Dataset & Prefill Len & Decode Len & Prefill Speed (TPS) & Decode Speed (TPS) & Peak Mem (MB) & Total Energy (J) & Per-run Energy (J) & Per-run Drain ($\mu$Ah) \\
\midrule

% ===================== Qwen2.5-3B =====================
\multirow{3}{*}{\makecell{Qwen2.5\\-3B-it}}
& HellaSwag    & 64  & 42  & 577.61  & 29.23 & 2217.2 & 137.7972 & 27.5594 & 1800 \\
& Persona-Chat & 536 & 46  & 775.12  & 26.18 & 2219.7 & 307.4400 & 61.4880 & 4000 \\
& DroidTask    & 741 & 3   & 808.10  & 21.06 & 2216.3 & 184.7    & 36.9    & 2400 \\
\midrule

% ===================== Llama3.2-3B-it =====================
\multirow{3}{*}{\makecell{Llama3.2\\-3B-it}}
& HellaSwag    & 64  & 42  & 674.29  & 28.04 & 2435.5 & 339.2928 & 67.8586 & 4400 \\
& Persona-Chat & 536 & 46  & 860.99  & 27.81 & 2433.1 & 415.0440 & 83.0088 & 5400 \\
& DroidTask    & 741 & 3   & 827.60  & 22.68 & 2433.6 & 538.0200 & 67.2525 & 4375 \\
\midrule

% ===================== SmolLM3-1.7B-it =====================
\multirow{3}{*}{\makecell{SmolLM2\\-1.7B-it}}
& HellaSwag    & 64  & 42  & 956.35  & 48.86 & 1119.4 & 185.8464 & 23.2308 & 1500 \\
& Persona-Chat & 536 & 46  & 1329.91 & 44.40 & 1119.1 & 402.9480 & 50.3685 & 3250 \\
& DroidTask    & 741 & 3   & 1343.98 & 37.00 & 1118.6 & 47.2500  & 9.4500  & 600 \\
\midrule

% ===================== Qwen3-1.7B-it =====================
\multirow{3}{*}{\makecell{Qwen3\\-1.7B}}
& HellaSwag    & 64  & 42  & 826.20  & 43.97 & 1564.6 & 232.2000 & 46.4400 & 3000 \\
& Persona-Chat & 536 & 46  & 1213.75 & 42.22 & 1567.9 & 171.2304 & 34.2461 & 2200 \\
& DroidTask    & 741 & 3   & 1177.22 & 34.27 & 1568.4 & 79.2360  & 15.8472 & 1000 \\

\bottomrule
\end{tabular}
\end{adjustbox}
\end{subtable}

\vspace{2mm} % 控制两个子表之间的间距

% ================= (b) =================
\begin{subtable}[t]{\textwidth}
\centering
\caption{Quant.npu: 4-bit per-channel weight quantization}

\begin{adjustbox}{max width=0.97\textwidth}
\begin{tabular}{c c | c c c c | c c c c} 
\toprule
Model & Dataset & Prefill Len & Decode Len & Prefill Speed (TPS) & Decode Speed (TPS) & Peak Mem (MB) & Total Energy (J) & Per-run Energy (J) & Per-run Drain ($\mu$Ah) \\
\midrule

% ===================== Qwen2.5-3B =====================
\multirow{3}{*}{\makecell{Qwen2.5\\-3B-it}}
& HellaSwag    & 64  & 42  & 693.13  & 29.23 & 1876.60 & 136.1510 & 27.2302 & 1780 \\
& Persona-Chat & 536 & 46  & 930.14  & 26.18 & 1877.85 & 293.1382 & 58.6276 & 3815 \\
& DroidTask    & 741 & 3   & 969.72  & 21.06 & 1876.15 & 158.0772 & 31.5815 & 2055 \\
\midrule

% ===================== Llama3.2-3B-it =====================
\multirow{3}{*}{\makecell{Llama3.2\\-3B-it}}
& HellaSwag    & 64  & 42  & 809.15  & 28.04 & 1985.75 & 335.9228 & 67.1846 & 4355 \\
& Persona-Chat & 536 & 46  & 1033.19 & 27.81 & 1984.55 & 399.1120 & 79.8224 & 5195 \\
& DroidTask    & 741 & 3   & 993.12  & 22.68 & 1984.80 & 461.2686 & 57.6586 & 3750 \\
\midrule

% ===================== SmolLM2-1.7B-it =====================
\multirow{3}{*}{\makecell{SmolLM2\\-1.7B-it}}
& HellaSwag    & 64  & 42  & 1147.62 & 48.86 & 994.70 & 183.6942 & 22.9618 & 1485 \\
& Persona-Chat & 536 & 46  & 1595.89 & 44.40 & 994.55 & 386.7214 & 48.3402 & 3120 \\
& DroidTask    & 741 & 3   & 1612.78 & 37.00 & 994.30 & 40.5061  & 8.1012  & 515 \\
\midrule

% ===================== Qwen3-1.7B-it =====================
\multirow{3}{*}{\makecell{Qwen3\\-1.7B}}
& HellaSwag    & 64  & 42  & 991.44  & 43.97 & 1217.30 & 229.2783 & 45.8557 & 2960 \\
& Persona-Chat & 536 & 46  & 1456.50 & 42.22 & 1218.95 & 164.2185 & 32.8437 & 2110 \\
& DroidTask    & 741 & 3   & 1412.66 & 34.27 & 1219.20 & 68.1084  & 13.6217 & 860 \\

\bottomrule
\end{tabular}
\end{adjustbox}
\end{subtable}

\end{table*}

\subsection{Deployment Performance on SM8750 NPU}
While the primary evaluation in ~\cref{End-to-end Latency} focuses on the Qualcomm SM8650 NPU, this section extends our hardware benchmarking to the newer Qualcomm SM8750 NPU. To provide a comprehensive profile of \sys's hardware efficiency, we report the inference speed, memory footprint, and energy consumption, and compare these results with the performance of the slower ExecuTorch-W4A16 on the same hardware.

\noindent \textbf{Experimental Setup.}
We evaluate four representative LLMs: Qwen2.5-3B-Instruct, Llama3.2-3B-Instruct, SmolLM2-1.7B-Instruct, and Qwen3-1.7B. We conduct evaluations on three datasets with distinct workload characteristics: HellaSwag, Persona-Chat, and DroidTask. Detailed results are summarized in ~\cref{tab:llm_speed_energy}.

\noindent \textbf{Inference Speed.}
The acceleration improvements observed on the SM8750 are highly consistent with those on the SM8650. \sys continues to demonstrate substantial speedup benefits, particularly during the prefill stage. For example, on the DroidTask dataset, Qwen2.5-3B-Instruct achieves a prefill throughput of 969.72 tokens/s, while SmolLM2-1.7B-Instruct reaches 1612.78 tokens/s.

\noindent \textbf{Memory Footprint.}
\sys consistently achieves lower memory consumption compared with baseline methods such as ExecuTorch. Although the peak memory footprint of on-device LLM inference is primarily dominated by model weights, \sys further reduces activation memory through lower-bit activation quantization. Specifically, \sys adopts W4A8 quantization, whereas the baseline uses W4A16 quantization. The lower activation precision therefore leads to a reduced overall memory footprint. As shown in ~\cref{tab:llm_speed_energy}, the peak memory usage of 3B models remains around 1876--1985 MB, while 1.7B models require only around 994--1219 MB, demonstrating the efficiency of our deployment scheme.

\noindent \textbf{Energy Consumption.}
Compared with ExecuTorch-W4A16, \sys also achieves lower energy consumption during inference. The energy reduction mainly comes from two factors. First, \sys more efficiently utilizes the computational capability of the underlying NPU hardware. Second, the reduced inference latency, especially during the prefill stage, shortens the duration of high-power execution states. For example, Qwen2.5-3B-Instruct consumes only 27.23 J per run on the HellaSwag dataset, confirming the exceptional suitability of \sys for power-constrained edge deployment.

\section{Limitations and Broader Impacts}
\label{Appendix: Limitations and Broader Impacts}

\noindent \textbf{Limitations}
While \sys significantly advances fully static quantization for mobile NPUs, it has two primary limitations. First, although matrix multiplications, which dominate inference latency, are already implemented in low-bit precision, the framework still retains 16-bit precision for a minor subset of operations (e.g., SiLU activations) to strictly preserve model perfomance. This introduces latency bottlenecks compared to a fully low-bit execution pipeline. Further reducing reliance on 16-bit operations without incurring non-negligible accuracy degradation remains a critical direction for future research. Second, \sys currently does not employ a dedicated strategy for curating the calibration dataset. Since our parameter optimization process is highly sensitive to activation distributions, the quality of the calibration set plays a critical role. Designing principled methods to select more representative calibration data could further improve the performance of quantized models.

\noindent \textbf{Broader Impacts}
The development of \sys addresses a key challenge in deploying large language models on mobile and edge devices, where computational resources and energy budgets are limited. This framework has practical implications for a wide range of real-world applications, including privacy-preserving on-device AI, mobile natural language understanding and generation, and low-latency edge computing. Through its enhanced efficiency and accuracy, \sys opens new possibilities for deploying powerful LLMs in resource-constrained environments, fostering the growth of AI in mobile and edge technologies.

%%%%%%%%%%%%%%%%%%%%%%%%%%%%%%%%%%%%%%%%%%%%%%%%%%%%%%%%%%%%

\end{document}